\documentclass[11pt,letterpaper]{article}
\usepackage[T1]{fontenc}
\usepackage[utf8]{inputenc}
\usepackage{lmodern}
\usepackage[margin=1in]{geometry}
\usepackage{amsmath,amssymb,amsthm,mathtools}
\usepackage{booktabs,array,enumitem,microtype}
\usepackage{xcolor,graphicx,float,placeins}
\usepackage{algorithm}
\usepackage[noend]{algpseudocode}
\usepackage[round,authoryear]{natbib}
\usepackage{url}
\usepackage[hidelinks]{hyperref}
\hypersetup{pdftitle={Groupwise Distortion Guarantees for Preference-Based Alignment},
  pdfauthor={Jacob Brodkey, Roberto Tamez, Aaron Roth}}
\newcommand{\artifactlink}{\url{https://github.com/Rob10001/groupwise-distortion-experiments}}
\newtheorem{theorem}{Theorem}
\newtheorem{lemma}[theorem]{Lemma}
\newtheorem{proposition}[theorem]{Proposition}

\theoremstyle{definition}
\newtheorem{definition}[theorem]{Definition}
\newtheorem{assumption}[theorem]{Assumption}
\newtheorem{example}[theorem]{Example}
\theoremstyle{remark}

\newcommand{\E}{\mathbb E}
\newcommand{\Ek}{\mathbb E_k}
\newcommand{\Et}{\mathbb E_{t-1}}
\newcommand{\R}{\mathbb R}
\newcommand{\X}{\mathcal X}
\newcommand{\F}{\mathcal F}
\newcommand{\B}{\mathcal B}
\newcommand{\simplex}{\Delta(\mathcal A)}
\newcommand{\piref}{\pi_{\mathrm{ref}}}
\newcommand{\pihat}{\widehat\pi}
\newcommand{\piGLHF}{\pi_{\mathrm{GLHF}}}
\newcommand{\ghat}{\widehat g}
\newcommand{\KL}{\mathrm{KL}}
\newcommand{\Expl}{\operatorname{Expl}}
\newcommand{\Dist}{\operatorname{Dist}}
\newcommand{\Var}{\operatorname{Var}}

\DeclareMathOperator*{\argmin}{arg\,min}
\DeclareMathOperator*{\argmax}{arg\,max}
\setlist[itemize]{leftmargin=1.4em,itemsep=0.2em,topsep=0.4em}
\setlist[enumerate]{leftmargin=1.7em,itemsep=0.2em,topsep=0.4em}

\title{Groupwise Distortion Guarantees for Preference-Based Alignment}
\author{Jacob Brodkey\thanks{Jacob Brodkey and Roberto Tamez contributed equally and are co-first authors.}
\quad Roberto Tamez\footnotemark[1]\quad Aaron Roth\\[0.6ex]
University of Pennsylvania}
\date{}

\begin{document}
\maketitle
\begin{abstract}
Preference-based alignment methods such as reinforcement learning from human feedback (RLHF) and Nash learning from human feedback (NLHF) aggregate pairwise preferences to learn an LLM policy, but a natural goal is maximizing \emph{social welfare} (average cardinal utility), which comparisons alone need not identify. Gölz, Haghtalab, and Yang (GHY) measure the gap by \emph{distortion}: the worst-case ratio between the welfare of the best fixed lottery (distribution over responses) and of the learned lottery. They show NLHF is optimal when every user receives the same lottery. Account-based LLMs, however, have information about their users and can serve different lotteries to different people. We give an efficient algorithm, GLHF, that learns a single group-conditioned policy from one comparison per user. Under individual Bradley--Terry comparisons, GLHF asymptotically matches GHY's optimal population distortion bound simultaneously on every group in a prespecified, possibly overlapping collection, with sample complexity growing logarithmically in the number of groups and inversely with the smallest group mass. A sharper guarantee for groups with similar preferences approaches distortion of one when members share a feasible favorite response. In experiments using human coffee ratings and synthetic LLM-generated ratings, GLHF lowers distortion in every evaluated group and substantially reduces worst-group distortion relative to NLHF and other group-agnostic baselines.
\end{abstract}

\section{Introduction}
\label{sec:introduction}
Preference-based alignment methods, such as reinforcement learning from human feedback (RLHF) \citep{ouyang2022} and Nash learning from human feedback (NLHF) \citep{munos2024}, train LLM policies using pairwise comparisons of candidate outputs. A natural goal is to maximize \emph{social welfare}: the average cardinal utility users obtain from the responses they receive~\citep{golz2025}. However, when users have heterogeneous preferences~\citep{santurkar2023,kirk2024prism}, aggregate choice probabilities need not identify average response utilities.

G\"olz, Haghtalab, and Yang (GHY)~\citep{golz2025} study this gap through \emph{distortion}: the worst-case ratio of the welfare of the best fixed \emph{lottery} (a distribution over responses), chosen with full knowledge of utilities, to the welfare of a lottery learned only from sampled comparisons. Assuming individual Bradley--Terry preferences (each user's comparison probabilities depend logistically on their utility differences) and one comparison per sampled user, they prove a worst-case lower bound on distortion and show that NLHF matches it asymptotically, even when the lottery must remain close to a reference distribution. The NLHF lottery is optimized over the population, so its welfare guarantee does not extend to  subgroups of users.

Account-based LLMs, however, observe user context such as custom instructions, settings, and persistent memory~\citep{openai2023custom,openai2024memory}, so a \emph{policy} can serve different lotteries to different users. We ask whether GHY's optimal distortion guarantee can be obtained simultaneously on prespecified, possibly overlapping groups of users, at sample and computational cost growing only modestly with the number of groups. For disjoint groups this is easy: separately solve an NLHF problem on each group. With overlapping groups, a user belongs to several groups but can receive only one lottery, and so this approach fails. Learning a separate lottery for each group intersection pattern recovers disjointness, but can create up to $2^K$ learning problems and fragment the data across those intersections.

\subsection{Our model and contribution}
A fixed prompt has $m$ candidate completions. Each user's observable context reveals membership in $K$ possibly overlapping groups, and we learn a group-conditioned policy from $T$ independently sampled users, each giving one pairwise comparison among completions. The queried completion pair is either uniformly random (\emph{passive} sampling) or consists of one completion drawn from the learner's current policy and one drawn uniformly from the rest (\emph{active} sampling). Following GHY, every served lottery and every comparator must lie in a Kullback--Leibler (KL) ball of radius $\tau$ around a reference lottery $\piref$; a uniform $\piref$ with $\tau=\log m$ recovers the unconstrained ``social-choice'' setting. Each group's benchmark is its own best feasible fixed lottery.

\paragraph{Simultaneous groupwise comparison guarantees.}
We introduce Groupwise Learning from Human Feedback (\textsc{GLHF}), which learns a single feasible group-conditioned policy $\piGLHF$ such that, simultaneously for every group, no feasible lottery is substantially preferred to it on average over the group. This is an approximate groupwise analogue of the maximal-lottery property~\citep{fishburn1984} underlying NLHF. Let the \emph{exploitability} of $\pi$ on group $k$, $\Expl_k(\pi)$, be the average preference margin (probability of winning minus probability of losing) of the best challenger lottery over $\pi$ on group $k$. Under active sampling, with probability at least $1-\delta$, simultaneously for every group $k$,
\[
\Expl_k(\piGLHF)\le\varepsilon_{T,k},
\qquad
\varepsilon_{T,k}
=
O\!\left(
\sqrt{
\frac{(m-1)\bigl(1+\tau+\log(K/\delta)+\log\log(eT)\bigr)}
     {Tp_k}
}
\right)
\quad\text{as }Tp_k\to\infty,
\]
where $p_k$ is the population mass of group $k$. Passive sampling replaces $m-1$ by $\binom m2$. GLHF adapts its learning rates to the observed comparisons; it does not need to know the group masses.

\paragraph{Groupwise welfare.}
Under individual Bradley--Terry preferences with inverse temperature $\beta>0$, these certificates give a simultaneous welfare guarantee for all groups. As the comparison error vanishes, the distortion bound for each group with positive benchmark welfare approaches
\[
\frac{\beta}{2}\,\frac{e^\beta+1}{e^\beta-1}.
\]
This asymptotic distortion factor approaches one as $\beta\downarrow0$ and grows like $\beta/2$ as $\beta\to\infty$. It is the optimal population-level distortion factor of \citet{golz2025}. Larger $\beta$ makes comparisons more reliably follow utility order; in the deterministic limit, they no longer reveal the intensity of a user's preference. Theorem~\ref{thm:barrier} shows that the same factor is worst-case optimal for contextual policies when observable context is uninformative.

\paragraph{Beyond worst-case distortion.} Our welfare bound above holds in the worst case over \emph{all} possible group structures.
But groups that are informative about their members' preferences can yield substantially stronger guarantees. Suppose a group's best feasible lottery gives its members nearly the utility they would obtain from their individual favorite completions. Theorem~\ref{prop:homogeneous} shows that GLHF then nearly attains that group's benchmark welfare, with a loss proportional to this shortfall plus statistical error. Different groups can agree on different completions, and they need not partition the population. In particular, when almost every member of a group shares a favorite completion whose point mass is feasible, the group's distortion approaches one. More generally, approximate agreement suffices when the KL constraint does not itself impose a large welfare loss.

\paragraph{Experiments.}
Using human ratings of four coffees, with training limited to one simulated comparison per respondent and held-out ratings used to measure welfare directly, multigroup GLHF achieves mean worst-group distortion $1.08$ versus $1.28$ for NLHF. A parallel exploratory study uses synthetic, LLM-generated ratings of six answers to one question. Section~\ref{sec:experiments} reports both studies and compares active and passive sampling as the number of comparisons grows. Our code, input datasets, and saved results are publicly available.\footnote{\artifactlink{}}

\paragraph{Connections to prior work.}
Multigroup learning asks for a predictor that competes with each overlapping group's own best predictor~\citep{blum2020,rothblum2021,tosh2022,globusharris2022,acharya2024}. GLHF pursues an analogous guarantee for preference-based alignment: within every group, no feasible fixed lottery should substantially outperform the learned policy in pairwise comparisons. Personalized preference methods also adapt predictions or policies to users and groups~\citep{zhao2024gpo,ramesh2024}, including by minimizing the largest group-specific preference-training loss. We obtain a stronger groupwise-optimality guarantee rather than only a minimax guarantee over groups, and do so with respect to cardinal welfare rather than only training loss.  Section~\ref{sec:related} explains these differences in objectives and benchmarks and discusses other related work.

\section{Preliminaries}
\label{sec:model}

Fix a prompt and $m\ge2$ candidate responses, called \emph{alternatives}, indexed by $\mathcal A:=\{1,\ldots,m\}$. A user is drawn from a fixed population distribution over an observable context $X$ in a measurable space $\X$ and an unobserved utility vector $U\in[0,1]^m$. Context and utility may be arbitrarily correlated. All population expectations below refer to this joint distribution.

\paragraph{Lotteries, policies, and groups.}
For clarity, we distinguish a \emph{policy} from a \emph{lottery}. Both specify distributions over $\mathcal A$, but a policy may depend on context, whereas a lottery is fixed. Formally, a lottery is a vector $q\in\simplex:=\{q\in\R^m:q\ge0,\ \sum_aq(a)=1\}$, while a policy is a map $\pi:\X\to\simplex$. Write $\pi(a\mid x)$ for the probability the policy serves $a$ given a context $x$.

Before training, the learner is given $K\ge1$ possibly overlapping groups. Group $k$ has an indicator $G_k:\X\to\{0,1\}$ and population mass $p_k:=\Pr(G_k(X)=1)>0$; the mass need not be known to the learner. The groups cover the context space: $\sum_{k=1}^K G_k(x)\ge1$ for every $x$. This can always be ensured by including the universal group $G\equiv1$, which also gives a population-wide guarantee. The learned policy is \emph{group-conditioned}: it depends on $x$ only through $(G_1(x),\ldots,G_K(x))$. Group membership is available during both training and deployment.

\paragraph{A common feasible set.}
Before training, fix a full-support reference lottery $\piref$, meaning $\piref(a)>0$ for every $a$, and a radius $\tau\ge0$. For a full-support lottery $v$, the Kullback--Leibler (KL) divergence is $\KL(q\|v):=\sum_aq(a)\log(q(a)/v(a))$, with zero summands when $q(a)=0$. All logarithms are natural. The common feasible set is
\begin{equation}
\B_\tau(\piref)
:=\{q\in\simplex:\KL(q\|\piref)\le\tau\}.
\label{eq:ball}
\end{equation}
We abbreviate the set by $\B_\tau$. A policy is feasible if $\pi(x)\in\B_\tau$ for every $x$. Following GHY's terminology, the \emph{alignment setting} imposes this KL constraint, while the \emph{social choice setting} allows every lottery. The latter is recovered once $\tau\ge\log(1/\min_a\piref(a))$, so we cap $\tau$ at this value without changing the feasible set. At $\tau=0$, only $\piref$ is feasible and no learning is needed.

\paragraph{Utilities and welfare.}
For integrable $Z$, write $\Ek[Z]:=\E[Z\mid G_k(X)=1]=\E[G_k(X)Z]/p_k$. Group welfare and its best fixed-lottery benchmark are
\begin{equation}
W_k(\pi):=\Ek[\pi(X)^\top U],
\qquad
W_k^\star:=\max_{q\in\B_\tau}W_k(q).
\label{eq:welfare}
\end{equation}

Following GHY's welfare-ratio notion~\citep{golz2025}, for $W_k(\pi)>0$ define the \emph{constrained group distortion}
\begin{equation}
\Dist_k(\pi):=\frac{W_k^\star}{W_k(\pi)}.
\label{eq:group-distortion}
\end{equation}
\begin{example}[A contextual policy can outperform every fixed lottery]
Suppose there are two equally common, disjoint user types and two responses. Their utility vectors are $(1,0)$ and $(0,1)$, respectively. In the social choice setting, every fixed lottery has population welfare $1/2$. If group memberships reveal the type, a policy that serves each type its preferred response has welfare $1$. Its distortion on the universal group is therefore less than one ($1/2$) meaning it outperforms the best fixed lottery.
\end{example}

\paragraph{Pairwise feedback and margins.}
A comparison presents a user with two distinct alternatives $\{a,b\}$ and
records which one they choose. Write $a\succ b$ for the event that the user
chooses $a$ when this pair is queried; exactly one of $a\succ b$ and
$b\succ a$ occurs. Conditional on the user's $(X,U)$ and the queried pair,
the choice follows a fixed, possibly stochastic comparison law. The comparison
guarantee allows this law to depend arbitrarily on $(X,U)$; the welfare
guarantee below will impose a Bradley--Terry law. In probabilities involving
$a\succ b$, the query $\{a,b\}$ is implicit. Define the signed preference
margin at context $x$ by
\begin{equation}
M_{ab}(x):=\Pr(a\succ b\mid X=x) - \Pr(b \succ a\mid X=x)
=2\Pr(a\succ b\mid X=x)-1,
\label{eq:margin}
\end{equation}
and set $M_{aa}(x)=0$. The matrix $M(x)$ is antisymmetric with entries in $[-1,1]$. For independent draws from lotteries $q,p$, independent of the user's latent utility given $x$, the \emph{advantage} $q^\top M(x)p$ is the expected win-minus-loss margin of $q$ against $p$. The algorithm observes comparisons, not $M(x)$.

\begin{definition}[Signed groupwise exploitability]
\label{def:exploit}
For a policy $\pi$, its exploitability on group $k$ is
\begin{equation}
\Expl_k(\pi):=\sup_{q\in\B_\tau}\Ek[q^\top M(X)\pi(X)].
\label{eq:exploit}
\end{equation}
This is the largest group-average advantage of a fixed feasible lottery over $\pi$. We say $\pi$ has an $\varepsilon_k$ comparison certificate for group $k$ if $\Expl_k(\pi)\le\varepsilon_k$. A group-conditioned policy can outperform every fixed feasible lottery on average, so its exploitability can be negative.
\end{definition}

\paragraph{Training data and sampling protocols.}
Training uses $T\ge2$ independent users $(X_t,U_t)$ from the deployment population. Before round $t$, the learner fixes a current policy map $x\mapsto\pihat_t(x)$ using previous observations. It then observes $X_t$, queries one pair of distinct alternatives, and observes the winner, never $U_t$. Pair selection is independent of $U_t$ conditional on the history and $X_t$; the winner is then drawn from the comparison law above, independently of the history given $(X_t,U_t)$ and the pair. We study two ways to choose the pair:
\begin{itemize}
\item \emph{Passive uniform sampling} draws an unordered pair uniformly, independently of the current policy.
\item \emph{Policy-adaptive active sampling} draws $A\sim\pihat_t(X_t)$, then draws $B$ uniformly from $\mathcal A\setminus\{A\}$, and queries the unordered pair $\{A,B\}$.
\end{itemize}
Let $\rho_t(a,b)$ denote the probability of querying $\{a,b\}$ conditional on the history and $X_t$. Thus, for $a\ne b$,
\begin{equation}
\rho_t^{\mathrm{pas}}(a,b)=\binom m2^{-1},
\qquad
\rho_t^{\mathrm{act}}(a,b)=\frac{\pihat_t(a\mid X_t)+\pihat_t(b\mid X_t)}{m-1}.
\label{eq:pair-probabilities}
\end{equation}
The active protocol concentrates queries on pairs with a response the current policy is likely to serve.

\paragraph{A utility model for welfare.}
The margins above are defined for any pairwise response law. To relate them to welfare, we assume the following link between comparison probabilities and utilities.
\begin{assumption}[Individual Bradley--Terry comparisons]
\label{ass:bt}
There is a common inverse temperature $\beta>0$ such that, conditional on $X$, $U$, and a queried pair $\{a,b\}$, the probability of $a\succ b$ is $\sigma(\beta(U(a)-U(b)))$, where $\sigma(z)=(1+e^{-z})^{-1}$~\citep{bradleyterry1952}.
\end{assumption}
Consequently, $M_{ab}(x)=\E[2\sigma(\beta(U(a)-U(b)))-1\mid X=x]$. The latent utilities may vary arbitrarily across users.

\section{Simultaneous welfare guarantees}
\label{sec:results}

GLHF learns one feasible contextual policy such that, within every named
group, no fixed feasible lottery has much average pairwise advantage over it.
In each training round, GLHF combines the lotteries maintained for the sampled user's
groups and updates them from that user's single comparison; Section~\ref{sec:algorithm}
gives the procedure and its pseudocode.
We first bound this advantage from sampled comparisons, then relate it to
each group's welfare relative to its own best fixed feasible lottery.

\begin{theorem}[Simultaneous comparison guarantees]
\label{thm:certificate}
Under either sampling protocol in Section~\ref{sec:model}, define the protocol-dependent factor
\begin{equation}
R:=\begin{cases}
m-1,&\text{policy-adaptive active sampling},\\
\binom m2,&\text{passive uniform sampling}
\end{cases}.
\label{eq:protocol-scale}
\end{equation}
Define the learning-rate parameter
\begin{equation}
\Lambda_T:=\max\left\{1,\log\!\left[K\left(1+\tfrac12\log(T+1)\right)\right]\right\}.
\label{eq:adaptive-log-factor}
\end{equation}
For $\delta\in(0,1)$, let
\begin{equation}
H:=\tau+\Lambda_T+\log\frac{2K}{\delta},
\qquad
\varepsilon_{T,k}:=8\sqrt{\frac{RH}{Tp_k}}+8\frac{RH}{Tp_k}
\quad(\tau>0).
\label{eq:main-error}
\end{equation}
For $\tau>0$, Algorithm~\ref{alg:main} returns a group-conditioned policy $\piGLHF$
satisfying $\KL(\piGLHF(x)\|\piref)\le\tau$ for every context $x$.
With probability at least $1-\delta$, it satisfies simultaneously for every $k\in[K]$,
\begin{equation}
\Expl_k(\piGLHF)\le\varepsilon_{T,k}.
\label{eq:main-certificate}
\end{equation}
\end{theorem}

For a fixed target error, the protocol-dependent factor in the comparison
requirement is $m-1$ under active sampling and $\binom m2$ under passive sampling.
The quantity $Tp_k$ is the expected number of sampled users in group $k$.
The algorithm adapts its learning rates to observed scores and does not need
the masses $p_k$ or confidence level $\delta$. The number of groups enters
only logarithmically, and adaptation contributes an additional
$O(\log\log(T+1))$ term to $H$.
The comparison guarantee itself makes no utility-model assumption.

\paragraph{Why the certificate bounds welfare.}
Set $L:=\sigma'(0)=1/4$ and
$\ell_\beta:=(\sigma(\beta)-1/2)/\beta$. Write
$f(d):=2\sigma(\beta d)-1$. This function is odd and concave on $[0,1]$,
so its chord between $0$ and $1$ and its tangent at $0$ give
\begin{equation}
2\beta\ell_\beta d\le f(d)\le2\beta Ld,
\qquad d\in[0,1].
\label{eq:link-sandwich}
\end{equation}
These bounds yield the following linearization inequality
\citep[Lemma~1 and Appendix~C]{golz2025}:
\begin{equation}
2\sigma(\beta(u-v))-1
\le 2\beta\big(Lu-\ell_\beta v\big),
\qquad u,v\in[0,1].
\label{eq:asymmetric-link}
\end{equation}
For $u\ge v$, use the tangent bound and $L\ge\ell_\beta$; for $u<v$,
use oddness and the chord bound.
The left side is a comparison margin; the right side is linear in the two
utilities. Averaging within a group therefore relates its comparison margin
to two group welfares. For each group, we apply this inequality with the group's best fixed
feasible lottery chosen as the comparator.

\begin{proposition}[Simultaneous welfare guarantees]
\label{prop:welfare-main}
\label{prop:bridge}
Under Assumption~\ref{ass:bt}, on the $1-\delta$ success event of Theorem~\ref{thm:certificate},
the GLHF policy satisfies simultaneously for every $k\in[K]$,
\begin{equation}
W_k(\piGLHF)\ge\frac{\ell_\beta}{L}W_k^\star-\frac{\varepsilon_{T,k}}{2\beta L}.
\label{eq:main-welfare}
\end{equation}
\end{proposition}

\begin{proof}
Fix a group $k$ and a welfare-maximizing feasible lottery
$q_k^\star\in\argmax_{q\in\B_\tau}W_k(q)$. Draw a fresh user $(X,U)$ from
group $k$, then draw $A\sim\piGLHF(X)$ and $B\sim q_k^\star$ independently
of each other and of $U$ given $X$. All expectations below include these
response draws. Theorem~\ref{thm:certificate} and antisymmetry give
$-\varepsilon_{T,k}\le\Ek[\piGLHF(X)^\top M(X)q_k^\star]$. Thus
\begin{equation}
\begin{aligned}
-\varepsilon_{T,k}
&\le \Ek[\piGLHF(X)^\top M(X)q_k^\star]=\Ek[M_{AB}(X)]\\
&=\Ek\!\left[2\sigma\!\left(\beta(U(A)-U(B))\right)-1\right]\\
&\le 2\beta\,\Ek\!\left[LU(A)-\ell_\beta U(B)\right]\\
&=2\beta\big(LW_k(\piGLHF)-\ell_\beta W_k^\star\big).
\end{aligned}
\label{eq:certificate-welfare-sandwich}
\end{equation}
The second line uses the definition of $M$, the third applies
\eqref{eq:asymmetric-link}, and the last uses linearity of expectation:
$\Ek[U(A)]=W_k(\piGLHF)$ and $\Ek[U(B)]=W_k^\star$.
Rearranging gives \eqref{eq:main-welfare}.
\end{proof}

For every group with $W_k^\star>0$, define the normalized comparison error
$\Delta_k:=\varepsilon_{T,k}/(2\beta\ell_\beta W_k^\star)$.
Proposition~\ref{prop:welfare-main} gives the finite-sample distortion bound
\[
\Dist_k(\piGLHF)\le\frac{L/\ell_\beta}{1-\Delta_k},
\qquad \Delta_k<1.
\]
As the certificate error vanishes, this bound approaches
$L/\ell_\beta=\beta(1+e^{-\beta})/[2(1-e^{-\beta})]$.
The bound applies in both the social
choice and alignment settings, with each group's benchmark taken over
the same feasible set as the learned policy.

\subsection{Why the worst-case factor cannot be improved}
\label{sec:barrier}

The lower-bound construction of
\citet[Theorem~3 and Appendix~E.3]{golz2025} extends directly to the
groupwise setting. When utilities are independent of context, averaging a
contextual policy over any group's contexts gives a fixed lottery with the
same group welfare. The following consequence shows that the welfare factor
above is optimal over the full instance class.

\begin{theorem}[Worst-case optimality of the welfare factor]
\label{thm:barrier}
Fix $\beta>0$. For every $c>\ell_\beta/L$, there are a response-set size
$m$, a uniform reference, and a nonbinding radius $\tau=\log m$ such that no
algorithm can guarantee
\begin{equation}
\E[W_k(\widehat\pi)]\ge cW_k^\star
\label{eq:impossible-fraction}
\end{equation}
on every instance for any specified positive-mass group $k$. This holds at
every sample size, with adaptive pair selection, randomized contextual output,
uninformative contexts, and one comparison from each independently sampled
user. The expectation is over training and algorithmic randomness.
\end{theorem}

For completeness, Appendix~\ref{app:barrier} gives the proof using GHY's
indistinguishable instances.

For GLHF, let $\pi_T$ be its output after $T$ comparisons. Choosing
$\delta_T=1/T$ makes $\varepsilon_{T,k}\to0$ on each fixed instance with
$\tau>0$. Nonnegative welfare and
Proposition~\ref{prop:welfare-main} then give
$\liminf_{T\to\infty}\E[W_k(\pi_T)]\ge(\ell_\beta/L)W_k^\star$,
matching the theorem. 

\subsection{Beyond worst-case distortion: preference-aligned groups}
\label{sec:homogeneous}
\label{app:homogeneous}

The worst-case factor in Proposition~\ref{prop:welfare-main} permits arbitrary
preference heterogeneity within every group. A sharper groupwise bound is available
when a group's best fixed feasible lottery nearly attains its members'
individual-best utilities. For each group $k$, fix a welfare-maximizing
feasible lottery $q_k^\star\in\argmax_{q\in\B_\tau}W_k(q)$.
Different groups may have different maximizers.
We quantify the residual welfare loss within group $k$ by
$\mathrm{Res}_k:=\Ek[\max_a U(a)]-W_k^\star\ge0$.
The following result replaces a fixed multiplicative welfare loss by a loss
that shrinks with this residual. It applies separately to every given group,
even when other groups have heterogeneous preferences.

\begin{theorem}[Near-optimal welfare for preference-aligned groups]
\label{prop:homogeneous}
Under Assumption~\ref{ass:bt}, any policy $\pi$ satisfying
$\Expl_k(\pi)\le\varepsilon$ obeys
\begin{equation}
W_k(\pi)\ge W_k^\star-
\frac{\varepsilon}{2\beta\ell_\beta}
-\left(\frac{L}{\ell_\beta}-1\right)\mathrm{Res}_k.
\label{eq:homogeneous}
\end{equation}
In particular, GLHF's output satisfies this bound simultaneously
for every group with $\varepsilon=\varepsilon_{T,k}$ on the event of
Theorem~\ref{thm:certificate}.
\end{theorem}

\begin{proof}
We apply the same linearization inequality as in
Proposition~\ref{prop:welfare-main}, now to shortfalls from each user's
maximum utility.
Draw a user $(X,U)$ from group $k$, then draw
$A\sim\pi(X)$ and $B\sim q_k^\star$, independently of each other and of
$U$ given $X$. Write $U_{\max}:=\max_aU(a)$; all expectations include
the response draws. Apply \eqref{eq:asymmetric-link} with
$u=U_{\max}-U(B)$ and $v=U_{\max}-U(A)$, both in $[0,1]$.
Since $u-v=U(A)-U(B)$, the comparison certificate and antisymmetry give
\begin{align*}
-\varepsilon
&\le\Ek[\pi(X)^\top M(X)q_k^\star]
=\Ek\!\left[2\sigma\!\left(\beta(U(A)-U(B))\right)-1\right]\\
&\le2\beta\,\Ek\!\left[L(U_{\max}-U(B))
                     -\ell_\beta(U_{\max}-U(A))\right]\\
&=2\beta\left[\ell_\beta\big(W_k(\pi)-W_k^\star\big)
             +(L-\ell_\beta)\mathrm{Res}_k\right].
\end{align*}
Rearranging proves \eqref{eq:homogeneous}.
\end{proof}

For GLHF and $W_k^\star>0$, whenever the denominator is positive,
\[
\Dist_k(\piGLHF)
\le\frac{1}{1-\Delta_k-(L/\ell_\beta-1)\mathrm{Res}_k/W_k^\star}.
\]
A small residual can therefore give
a much sharper bound than the worst-case factor $L/\ell_\beta$.
To distinguish preference disagreement from the effect of the constraint, decompose
\begin{equation}
\mathrm{Res}_k=
\underbrace{\Ek[\max_aU(a)]-\max_{q\in\simplex}W_k(q)}_{\text{disagreement within the group}}
+\underbrace{\max_{q\in\simplex}W_k(q)-W_k^\star}_{\text{cost of the KL constraint}}.
\label{eq:residual-welfare-loss}
\end{equation}
Informative groupings can make the first term small; the benchmark policy
and radius determine the second. In the social choice setting, if
members of group $k$ share a favorite alternative, its point mass attains
every member's highest utility and $\mathrm{Res}_k=0$. The same conclusion
holds in the alignment setting whenever that point mass is feasible.
More generally, any fixed feasible lottery supported on each member's
utility-maximizing alternatives has zero residual loss.

For fixed $\beta>0$, the bound approaches distortion one
when both $\Delta_k$ and
$\mathrm{Res}_k/W_k^\star$ vanish. Approximate agreement is thus enough:
each group can have its own nearly optimal feasible lottery, and the
groups may overlap. Appendix~\ref{app:population-transfer} discusses
when such guarantees also control population welfare.

\section{Groupwise Learning from Human Feedback}
\label{sec:algorithm}

GLHF maintains one \emph{specialist} per group: a feasible lottery
$\pi_{t,k}\in\B_\tau$ and a positive outer weight $w_{t,k}$. Before each
training round, these states and their current learning rates define a policy
for every possible context.
The learner then observes one user's context and comparison and updates all
specialists whose groups contain that user. The deployed policy averages the
policies saved before the training rounds. Algorithm~\ref{alg:main} gives
the complete procedure. The outer weights control each specialist's
advantage over the current policy within its group. The inner updates
control the advantage of any fixed feasible lottery over that specialist.
The learning rates decrease as each group accumulates comparison scores.
Utilities, the Bradley--Terry parameter $\beta$, group
masses, and the confidence level are not inputs.

\paragraph{Policy and comparison score.}
For $\tau>0$, write $\alpha_{t,k}\in(0,1/2]$ for the outer learning rate
of group $k$, determined from previous scores by \eqref{eq:rates} below.
Given the state at round $t$, the current policy serves context $x$
according to
\begin{equation}
\pihat_t(x):=
\frac{\sum_{k:G_k(x)=1}\alpha_{t,k}w_{t,k}\pi_{t,k}}
{\sum_{k:G_k(x)=1}\alpha_{t,k}w_{t,k}}.
\label{eq:deployment}
\end{equation}
Every context belongs to at least one group, so the denominator is positive.
The mixture is feasible because the KL ball is convex, even when groups
overlap. After observing $X_t$, let $a$ and $b$ be the winner and loser of the
queried pair, and let $e_c$ be the $c$th coordinate vector. With pair
probability $\rho_t(a,b)$ from \eqref{eq:pair-probabilities} and protocol
scale $R$ from \eqref{eq:protocol-scale}, form
\begin{equation}
\ghat_t:=\frac{\pihat_t(b\mid X_t)e_a-\pihat_t(a\mid X_t)e_b}
{R\rho_t(a,b)}.
\label{eq:estimated-game}
\end{equation}
The pair probability corrects for how the query was selected; $R$
normalizes the score. Let $\F_{t-1}$ denote the training history before
round $t$. Conditional on $\F_{t-1}$ and $X_t$, the score satisfies
\begin{equation}
\begin{gathered}
\pihat_t(X_t)^\top\ghat_t=0,
\qquad
\E[\ghat_t\mid\F_{t-1},X_t]
=\frac{M(X_t)\pihat_t(X_t)}{R},\\
|\ghat_{t,c}|\le1,
\qquad
\E[\ghat_{t,c}^2\mid\F_{t-1},X_t]\le\frac1R
\quad\text{for every }c\in\mathcal A.
\end{gathered}
\label{eq:score-interface}
\end{equation}
Thus the current policy has exactly zero estimated advantage over itself,
while a specialist's score estimates its normalized pairwise advantage over
that policy. Section~\ref{sec:score-moments} derives the second-moment bound
and explains its effect on the active and passive rates;
Appendix~\ref{app:feedback} verifies the other properties, including
the score-span bound
\begin{equation}
\max_a\ghat_{t,a}-\min_a\ghat_{t,a}\le1.
\label{eq:score-span}
\end{equation}

\paragraph{Adapting the learning rates.}
For each group containing the current user (called \emph{active} on that
round), the specialist's mean score and its variance over
alternatives are
\begin{equation}
r_{t,k}:=\pi_{t,k}^\top\ghat_t,
\qquad
c_{t,k}:=\sum_a\pi_{t,k}(a)(\ghat_{t,a}-r_{t,k})^2.
\label{eq:adaptive-moments}
\end{equation}
The outer update uses the scalar score $r_{t,k}$; the inner update uses
differences between alternatives, measured by $c_{t,k}$. Track their
accumulation before round $t$ as
\[
Q_{t,k}:=\sum_{s<t:G_k(X_s)=1}r_{s,k}^2,
\qquad
C_{t,k}:=\sum_{s<t:G_k(X_s)=1}c_{s,k}.
\]
Starting from $Q_{1,k}=C_{1,k}=0$, set
\begin{equation}
\alpha_{t,k}:=\min\left\{\frac12,
\sqrt{\frac{\Lambda_T}{2(1+Q_{t,k})}}\right\},
\qquad
\eta_{t,k}:=\min\left\{1,
\sqrt{\frac{\tau}{2(1/4+C_{t,k})}}\right\},
\label{eq:rates}
\end{equation}
where $\Lambda_T$ is defined in \eqref{eq:adaptive-log-factor}. A group
with small accumulated scores takes larger steps. Section~\ref{sec:analysis}
shows how these inverse-square-root choices balance the regret terms.

\paragraph{Two updates.}
After observing the comparison, each active group adds $r_{t,k}^2$ and
$c_{t,k}$ to its accumulators and computes its next rates. Let
$\gamma_{t,k}^{\rm out}:=\alpha_{t+1,k}/\alpha_{t,k}$ and
$\gamma_{t,k}^{\rm in}:=\eta_{t+1,k}/\eta_{t,k}$, both in $(0,1]$.
The outer update is
\begin{equation}
w_{t+1,k}=(1/K)^{1-\gamma_{t,k}^{\rm out}}
\left[w_{t,k}(1+\alpha_{t,k}r_{t,k})\right]^{\gamma_{t,k}^{\rm out}}.
\label{eq:outer-update}
\end{equation}
The inner update forms a normalized lottery $z_{t,k}$ and projects it:
\begin{equation}
z_{t,k}(a)\propto
\pi_{t,k}(a)^{\gamma_{t,k}^{\rm in}}
\piref(a)^{1-\gamma_{t,k}^{\rm in}}e^{\eta_{t+1,k}\ghat_{t,a}},
\qquad
\pi_{t+1,k}=\argmin_{q\in\B_\tau}\KL(q\|z_{t,k}).
\label{eq:inner-update}
\end{equation}
When a learning rate decreases, the power $\gamma$ pulls the corresponding
state toward its initialization; when it stays fixed, $\gamma=1$ and these
are multiplicative-weight and exponentiated-gradient updates
\citep{gaillard2014,kivinen1997,beck2003}. The power adjustments allow
their regret bounds to telescope despite changing rates
(Appendix~\ref{app:proof-details}). All scores use the states from
the start of the round. Inactive states and rates do not change. The KL
projection reduces to a one-dimensional search (Appendix~\ref{app:projection}).

\begin{algorithm}[htbp]
\caption{Groupwise Learning from Human Feedback (GLHF)}
\label{alg:main}
\label{alg:body}
\begin{algorithmic}[1]
\Require $T\ge2$, covering indicators $G_1,\ldots,G_K$,
full-support reference $\piref$, radius $\tau\ge0$, passive or
active comparison sampling protocol
\If{$\tau=0$}
  \State \Return $\piGLHF(x)=\piref$ for every $x$.
\EndIf
\State Cap $\tau$ at $\log(1/\min_a\piref(a))$ and set $R$ by \eqref{eq:protocol-scale}.
\State Set $\Lambda_T$ by \eqref{eq:adaptive-log-factor}. For every $k$, initialize
$w_{1,k}=1/K$, $\pi_{1,k}=\piref$, $Q_{1,k}=C_{1,k}=0$, and rates by \eqref{eq:rates}.
\For{$t=1,\ldots,T$}
  \State Save the current state, which fixes the policy $x\mapsto\pihat_t(x)$
  by \eqref{eq:deployment}.
  \State Observe $X_t$ and sample a pair using the chosen protocol in
  Section~\ref{sec:model}.
  \State Observe the winner and loser and form $\ghat_t$ by \eqref{eq:estimated-game}.
  \For{each group $k$ with $G_k(X_t)=1$}
    \State Compute $r_{t,k},c_{t,k}$ by \eqref{eq:adaptive-moments}, set
    $Q_{t+1,k}=Q_{t,k}+r_{t,k}^2$, $C_{t+1,k}=C_{t,k}+c_{t,k}$,
    and compute the next rates by \eqref{eq:rates}.
    \State Update $w_{t+1,k}$ and $\pi_{t+1,k}$ by
    \eqref{eq:outer-update}--\eqref{eq:inner-update}.
  \EndFor
  \State Leave inactive groups unchanged.
\EndFor
\State \Return $\displaystyle\piGLHF(x)=\frac1T\sum_{t=1}^T\pihat_t(x)$.
\end{algorithmic}
\end{algorithm}

Each saved policy, and hence their pointwise average, takes values in
$\B_\tau$. At deployment, a new user's group memberships determine the
response lottery $\piGLHF(x)$; no further comparisons are needed. Saving a
state retains the complete map in \eqref{eq:deployment}, so each of the
$T$ policies can be evaluated at the new context, including contexts that
did not appear during training.

\paragraph{Computation.}
On each round, GLHF updates the weight and lottery for each group
containing the sampled user. If a numerical implementation uses at most
$J$ one-dimensional search steps per KL projection, each costing $O(m)$,
the updates take $O(Km(1+J))$ arithmetic operations per round in the worst
case, excluding the cost of evaluating group indicators. The current weights
and lotteries use $O(Km)$ space; saving all $T$ states uses $O(TKm)$ space,
and explicitly evaluating the average policy at a new context costs
$O(TKm)$ operations. To sample a response from that average, it suffices
to choose a saved round uniformly and then sample from its policy at the
new context.
\FloatBarrier

\section{Why the comparison guarantee holds}
\label{sec:analysis}

The proof has two ingredients. First, the sampling protocol determines the
second moment of the comparison score and hence its statistical error.
Second, the outer and inner updates use that score to control each group's
advantage against every fixed feasible lottery. We assume $\tau>0$ throughout
this section.

\subsection{Why adapting the sampled pair helps}
\label{sec:score-moments}

Condition on the history $\F_{t-1}$ and the current context $X_t$, and write
$\pi=\pihat_t(X_t)$ and $\rho(c,d)=\rho_t(c,d)$. Fix a coordinate $c$.
The score $\ghat_{t,c}$ is zero unless the sampled pair contains $c$.
If the pair is $\{c,d\}$, its magnitude is $\pi(d)/(R\rho(c,d))$,
whichever alternative wins. Therefore,
\begin{equation}
\begin{aligned}
\E[\ghat_{t,c}^2]
&=\frac1{R^2}\sum_{d\ne c}\frac{\pi(d)^2}{\rho(c,d)}\\
&=\begin{cases}
\displaystyle\frac1R\sum_{d\ne c}
\frac{\pi(d)^2}{\pi(c)+\pi(d)},&\text{active sampling},\\[1.2ex]
\displaystyle\frac1R\sum_{d\ne c}\pi(d)^2,&\text{passive sampling}
\end{cases}\\
&\le\frac1R\sum_{d\ne c}\pi(d)\le\frac1R.
\end{aligned}
\label{eq:score-second-moment}
\end{equation}
Here we substituted the pair probabilities in \eqref{eq:pair-probabilities}.
All denominators are positive because GLHF's current policy has full support;
Appendix~\ref{app:feasibility-invariant} verifies this invariant.

The estimator $R\ghat_t$ of the margin vector $M(X_t)\pi$ has
coordinate second moments at most $R$. The analysis below converts
this bound into simultaneous groupwise comparison guarantees,
with leading error of order $\sqrt{R/(Tp_k)}$ up to confidence
and complexity factors.
\subsection{Regret on overlapping groups}

One observed comparison supplies a score to every group containing its user.
The proof separates two issues: can a group's specialist outperform the
current policy by much, and can a fixed feasible lottery outperform that
specialist? Fix a group $k$ and a comparator $q\in\B_\tau$. Their cumulative
sampled advantage decomposes as
\begin{equation}
A_{T,k}(q):=\sum_{t:G_k(X_t)=1}q^\top\ghat_t
=\sum_{t:G_k(X_t)=1}\pi_{t,k}^\top\ghat_t
+\sum_{t:G_k(X_t)=1}(q-\pi_{t,k})^\top\ghat_t.
\label{eq:decomposition}
\end{equation}
The outer weights control the first sum; the inner update controls the
second. Their two score accumulators add up to one cumulative second moment:
\begin{equation}
V_{T,k}:=\sum_{t:G_k(X_t)=1}\sum_a\pi_{t,k}(a)\ghat_{t,a}^2
=Q_{T+1,k}+C_{T+1,k}.
\label{eq:second-moments}
\end{equation}
The following lemma summarizes their deterministic guarantees. The
log-weight and KL calculations are in Appendix~\ref{app:proof-details}.
\begin{lemma}[Outer and inner regret]
\label{lem:regret}
All outer weights remain positive. All specialist and current
policy lotteries have full support and lie in $\B_\tau$. For every group $k$
and every $q\in\B_\tau$,
\begin{align}
\sum_{t:G_k(X_t)=1}\pi_{t,k}^\top\ghat_t
&\le\frac{\Lambda_T}{\alpha_{T+1,k}}
+\sum_{t:G_k(X_t)=1}\alpha_{t,k}r_{t,k}^2,
\label{eq:outer-regret}\\
\sum_{t:G_k(X_t)=1}(q-\pi_{t,k})^\top\ghat_t
&\le\frac{\tau}{\eta_{T+1,k}}
+\sum_{t:G_k(X_t)=1}\eta_{t,k}c_{t,k}.
\label{eq:inner-regret}
\end{align}
Consequently,
\begin{equation}
A_{T,k}(q)\le
4\sqrt{(\tau+\Lambda_T)(V_{T,k}+5/4)}+2(\tau+\Lambda_T).
\label{eq:adaptive-regret}
\end{equation}
\end{lemma}
The key cancellation occurs before adjusting for a change in learning rate.
Let $\widetilde w_{t+1,j}=w_{t,j}(1+\alpha_{t,j}r_{t,j})$ on active
groups and $\widetilde w_{t+1,j}=w_{t,j}$ on inactive groups. Then
\begin{equation}
\sum_j(\widetilde w_{t+1,j}-w_{t,j})
=\sum_{j:G_j(X_t)=1}\alpha_{t,j}w_{t,j}\pi_{t,j}^\top\ghat_t
=0.
\label{eq:mass-cancellation}
\end{equation}
The last equality holds because the weighted mixture in
\eqref{eq:deployment} has zero score by \eqref{eq:score-interface}. Thus the comparison step preserves total weight. The subsequent
learning-rate adjustment can increase the total, but it remains
bounded by $1+\tfrac12\log(T+1)$. This bounds the final log-weight by $\Lambda_T$. The scaled log-weight
$\log(Kw_{t,k})/\alpha_{t,k}$ telescopes to give the outer bound.

For the inner bound, factor \eqref{eq:inner-update} into a
multiplicative-weights step, geometric interpolation toward $\piref$,
and KL projection. Interpolation accounts for rate changes: the potential
$(\KL(q\|\pi_{t,k})-\KL(q\|\piref))/\eta_{t,k}$ telescopes,
leaving the variance term $\eta_{t,k}c_{t,k}$. It starts at zero and ends
at least $-\tau/\eta_{T+1,k}$, yielding \eqref{eq:inner-regret}
(Appendix~\ref{app:inner-step}).

Each bound balances a complexity term divided by the final learning rate
against a sum of learning-rate-weighted squared scores. Inverse-square-root
rates make both terms grow as the square root of the accumulated scores.
The coordinate bound in \eqref{eq:score-interface} gives $r_{t,k}^2\le1$;
the span bound in \eqref{eq:score-span} gives $c_{t,k}\le1/4$. These bounds
explain the offsets and caps in \eqref{eq:rates}.
Combining the bounds using $Q_{T+1,k}+C_{T+1,k}=V_{T,k}$ gives
\eqref{eq:adaptive-regret}. Finally, the coordinate second-moment bound
\eqref{eq:score-second-moment} implies $\E[V_{T,k}]\le Tp_k/R$.
Concentration will replace $V_{T,k}$ by a bound of this order.

\subsection{From sampled scores to population margins}

Write $\Et[\cdot]:=\E[\cdot\mid\F_{t-1}]$. The history fixes the complete
policy $x\mapsto\pihat_t(x)$ before the independent context $X_t$ arrives.
Score unbiasedness therefore gives
$R\Et[G_k(X_t)q^\top\ghat_t]
=p_k\Ek[q^\top M(X)\pihat_t(X)]$,
where the right side evaluates this fixed policy on a fresh group-$k$ user.
The discrepancy between sampled and conditional-expected score vectors is
\begin{equation}
S_k:=\sum_{t=1}^T\bigl(
\Et[G_k(X_t)\ghat_t]-G_k(X_t)\ghat_t\bigr).
\label{eq:martingale}
\end{equation}
The output policy is the pointwise average
$\piGLHF(x)=T^{-1}\sum_{t=1}^T\pihat_t(x)$.
Summing the conditional score identity over rounds and using the definitions
of $A_{T,k}(q)$ and $S_k$ gives the exact population identity
\begin{equation}
Tp_k\Ek[q^\top M(X)\piGLHF(X)]
=R\bigl(A_{T,k}(q)+q^\top S_k\bigr).
\label{eq:population-conversion}
\end{equation}
The identity uses the saved policies at every possible context: it evaluates
each $\pihat_t$ on a fresh user before averaging, just as the output policy
does at deployment.

The KL radius plays two roles: it bounds a challenger's initial divergence
from the reference in the regret analysis, and it limits how much that
challenger can reweight the reference in the concentration argument.
The latter gives one sampling-error bound for every feasible challenger.
Appendix~\ref{app:concentration-proof} controls both $V_{T,k}$ and
$\sup_{q\in\B_\tau}q^\top S_k$ simultaneously for all groups. Substituting
these concentration bounds and the regret bounds with rates
\eqref{eq:rates} into \eqref{eq:population-conversion} gives, on one event
of probability at least $1-\delta$,
\[
\Ek[q^\top M(X)\piGLHF(X)]
\le 8\sqrt{\frac{RH}{Tp_k}}+8\frac{RH}{Tp_k}
=\varepsilon_{T,k}
\qquad\text{for every }k\text{ and }q\in\B_\tau,
\]
where $H=\tau+\Lambda_T+\log(2K/\delta)$ as in \eqref{eq:main-error}.
Taking the supremum over $q$ proves Theorem~\ref{thm:certificate};
Appendix~\ref{app:certificate-proof} contains the constant calculations.

\section{Empirical evaluation}
\label{sec:experiments}
We test whether learning a group-conditioned policy from pairwise comparisons improves welfare relative to learning a single lottery for everyone. We run two exploratory studies: one uses human ratings of four coffees, and the other uses synthetic ratings of six prompt completions generated by an LLM under different persona descriptions. Rescaled ratings serve as cardinal utilities, allowing us to generate training comparisons and evaluate the welfare of learned policies. The learner sees the comparisons and group memberships, but not the ratings.

Code, input datasets, saved results, and figure-generation scripts are available at \artifactlink{}. Appendix~\ref{app:empirical-code} describes the repository contents and reproduction commands.

\subsection{Experimental setup}
\paragraph{Methods and feedback.}
Our primary comparison is \emph{multigroup GLHF} against \emph{batch NLHF}, which estimates pairwise preference margins and finds a constrained maximal lottery to serve the whole population. We also evaluate \emph{pooled GLHF}, the same GLHF implementation with only the universal group, and the untrained uniform reference lottery. Comparing the two GLHF methods measures the benefit of group conditioning within the same learning procedure.

Each learner receives one comparison per training respondent. At the main setting, both GLHF methods use active sampling: the current policy selects the first alternative, and the second is drawn uniformly from the remaining alternatives. Batch NLHF uses passive comparisons, with each unordered pair equally likely. All learners use the same training respondents, although active methods may query different pairs. In the completion study, a respondent is a constructed row of synthetic ratings, as described below.

\paragraph{Constraints and comparison noise.}
The reference lottery is uniform. We set the learner radius $\tau_L$ and the fixed-lottery benchmark radius $\tau_E$ to $0.5$; these constrain each served lottery and each group's benchmark, respectively. At the main setting, comparisons follow rating order with random tie-breaking ($\beta=\infty$).  At this limit, the worst-case distortion factor $L/\ell_\beta$ diverges, and so the theory promises nothing; but as we will see the algorithm still performs well. We also give comparisons at finite $\beta$ values. 

The implementation uses the adaptive updates in Algorithm~\ref{alg:main}, choosing learning rates from observed scores without population group masses or a confidence parameter. Appendix~\ref{app:empirical-rates} gives numerical details. Section~\ref{sec:empirical-noise} evaluates finite Bradley--Terry noise, and Appendix~\ref{app:empirical-robustness} varies the learner radius.

\paragraph{Evaluation and group selection.}
For each held-out group, empirical distortion is
\begin{equation}
\widehat{\Dist}_k(\pi)=
\frac{\text{welfare of best fixed feasible lottery for group }k}
{\text{welfare of deployed policy }\pi\text{ on group }k}.
\label{eq:empirical-distortion-summary}
\end{equation}
We compute both quantities from the held-out ratings. Every group evaluates the same deployed policy against its own best fixed feasible lottery. A distortion of one means the policy matches that benchmark. A group-conditioned policy may achieve distortion below one by serving different lotteries to different members of the group.

Both studies reserve about 10\% of the rating rows for choosing groups, including groups whose preferences differ from the population average and groups whose preferences are more similar to it. Coffee groups were chosen by judgment, informed by the reserved sample; the synthetic study uses a fixed selection rule. Appendix~\ref{app:empirical-data} gives the definitions and selection procedure, including our earlier exposure to the Coffee data. The selected groups remain fixed during ten repeats of ten-fold cross-validation on the remaining data, with 50 learner seeds per split. Thus each fit trains on nine folds and evaluates welfare on the tenth. Worst-group distortion takes the maximum across groups within each run before averaging. Paired differences compare methods on the same split and seed. We report approximate 95\% intervals corrected for overlapping folds, conditional on the selected groups; Appendix~\ref{app:empirical-inference} gives the calculation.

\subsection{Coffee-rating study}
The Great American Coffee Taste Test~\citep{coffee2024} provides ratings of four coffees from 3,369 respondents after setting aside our group discovery sample. We evaluate six overlapping groups: the universal group and groups defined by age, coffee expertise, gender, having children, and working from home (Appendix~\ref{app:empirical-data}).

\begin{figure}[tbp]
\centering
\includegraphics[width=0.92\linewidth]{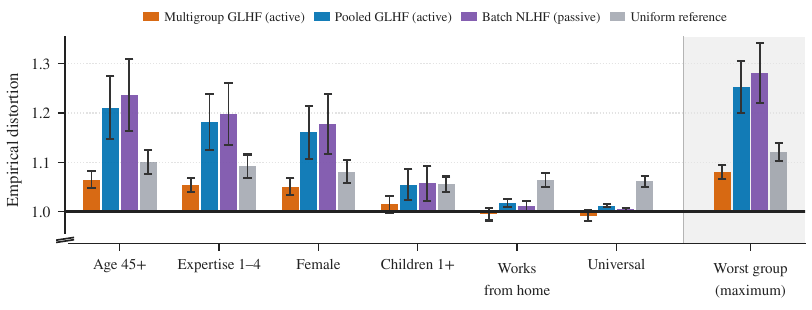}
\caption{Empirical distortion in the coffee-rating study at $\beta=\infty$ and $\tau_L=\tau_E=0.5$. Multigroup GLHF has the lowest estimated distortion in every selected group. The line at one marks each group's best fixed feasible lottery; lower is better. The rightmost cluster averages the within-run group maximum. Whiskers are approximate 95\% confidence intervals corrected for overlapping cross-validation folds.}
\label{fig:coffee-group-bars}
\end{figure}

Multigroup GLHF has lower estimated distortion than all three comparators in every group (Figure~\ref{fig:coffee-group-bars}). Mean within-run worst-group distortion is 1.08 for multigroup GLHF, compared with 1.28 for batch NLHF, 1.25 for pooled GLHF, and 1.12 for the uniform reference. The paired reduction relative to batch NLHF is $0.202$ (corrected 95\% interval $[0.151,0.253]$). Table~\ref{tab:empirical-anchor} reports paired differences for all three comparisons. Both population-wide learners raise population welfare relative to the reference lottery but lower welfare for several subgroups. As a result, the uniform reference has lower worst-group distortion than either population-wide learner. 

\subsection{Completion-rating study}
Claude Haiku 4.5 rated six fixed answers to \emph{Why is the sky blue?} under sixteen persona prompts crossing age, education, technical background, and preferred writing style. Each answer was rated separately. Independently drawn ratings within each persona were combined to form complete vectors of six utilities, which we call \emph{pseudo-respondents}. After reserving the discovery sample, the analysis contains 1,497 pseudo-respondents. Appendix~\ref{app:empirical-data} details the construction and group selection, and Appendix~\ref{app:persona-materials} gives the prompts and answers.

Mean within-run worst-group distortion is 1.07 for multigroup GLHF, compared with 1.36 for batch NLHF, 1.35 for pooled GLHF, and 1.74 for the uniform reference. The paired reduction relative to batch NLHF is $0.295$ $[0.282,0.308]$. The uniform reference performs much worse than in the coffee-rating study, in part because it puts one third of its mass on two factually incorrect answers that were rarely most preferred. Appendix~\ref{app:persona-results} reports the groupwise results.

\subsection{Finite comparison noise}
\label{sec:empirical-noise}
We also generate comparisons from the finite-noise Bradley--Terry model in Assumption~\ref{ass:bt}: alternative $a$ wins against $b$ with probability $\sigma(\beta[U(a)-U(b)])$. We hold the ratings and groups fixed, with $\tau_L=\tau_E=0.5$. Figure~\ref{fig:empirical-beta} shows $\beta\in\{1/2,1,2,4,\infty\}$ for passive GLHF and batch NLHF; active GLHF is evaluated at $4$ and $\infty$.

At $\beta=4$, active multigroup GLHF's mean worst-group distortion is $1.093$ in the coffee-rating study, versus $1.227$ for active pooled GLHF and $1.267$ for batch NLHF. In the completion-rating study the corresponding values are $1.059$, $1.318$, and $1.316$. The gains from group conditioning therefore persist when labels follow the finite-noise model used by the welfare analysis. Using the same passive comparison log, multigroup GLHF also outperforms passive pooled GLHF and batch NLHF, with distortions $1.105$ and $1.103$ in the two studies. This comparison separates the gains from the choice of sampling protocol.

Across the full noise grid, multigroup GLHF has lower mean worst-group distortion than the trained population-wide comparators in every tested setting, although passive multigroup's coffee estimates are slightly worse than uniform's at $\beta=1/2$ and $1$. Less label noise does not always improve welfare; active multigroup's completion-rating distortion is slightly lower at $\beta=4$ than at $\infty$ ($1.059$ versus $1.069$).

\begin{figure}[tbp]
\centering
\includegraphics[width=\linewidth]{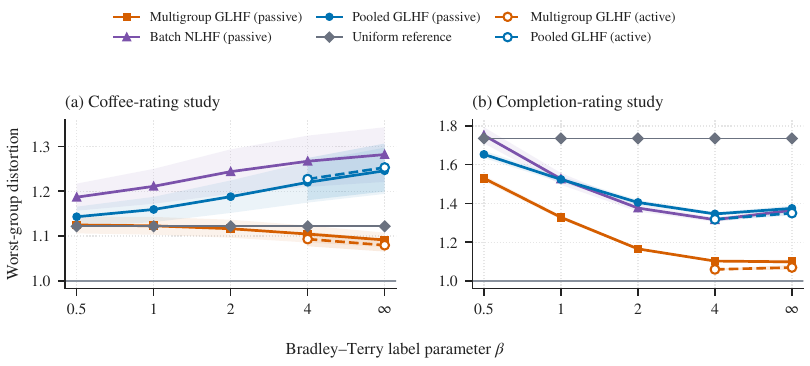}
\caption{Worst-group distortion under finite and deterministic comparison feedback, with $\tau_L=\tau_E=0.5$. Multigroup GLHF has the lowest mean distortion among the trained methods at every evaluated noise level in both studies. Larger $\beta$ makes labels more reliably follow rating order; at $\infty$, only ties are randomized. Passive GLHF and batch NLHF cover the full grid; active GLHF covers $4$ and $\infty$. Bands are corrected 95\% intervals for individual means, not simultaneous intervals over methods or noise levels.}
\label{fig:empirical-beta}
\end{figure}

\subsection{Active and passive sampling}
\label{sec:empirical-trajectories}
We next compare active and passive multigroup GLHF, holding the groups, rating data, and KL radii fixed. At $\beta=\infty$ and $\tau_L=\tau_E=0.5$, worst-group distortion falls from 1.0913 under passive sampling to 1.0795 under active sampling in the coffee-rating study, and from 1.0986 to 1.0695 in the completion-rating study. The paired reductions are $0.0118$ (corrected 95\% interval $[0.0071,0.0165]$) and $0.0291$ ($[0.0255,0.0327]$), respectively. This additional improvement is smaller than the gains over the population-wide learners.

Figure~\ref{fig:empirical-learning-curves} shows how held-out worst-group distortion changes as comparisons accumulate. At checkpoint $t$, we deploy the average of the first $t$ saved policies, each recorded before its comparison, and evaluate against the same held-out benchmark. The zero-comparison policy is uniform. Learning rates retain the planned full training horizon, so a checkpoint is a prefix of one full-length run, not a separately trained shorter run.

\begin{figure}[tbp]
\centering
\includegraphics[width=\linewidth]{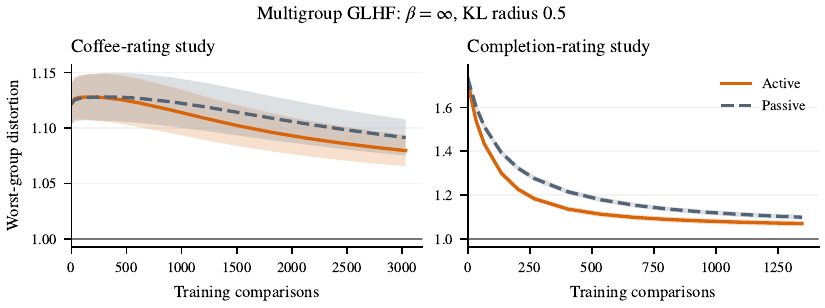}
\caption{Active and passive multigroup GLHF at $\beta=\infty$ and $\tau_L=\tau_E=0.5$. Active sampling gives lower final worst-group distortion in both studies, although early coffee checkpoints do not consistently favor it. Each checkpoint deploys the average of the first $t$ policies from a full-length adaptive run.  Shading gives corrected 95\% intervals at individual checkpoints. Intervals for the completion-rating study are narrow at this scale; vertical scales differ.}
\label{fig:empirical-learning-curves}
\end{figure}

In the completion-rating study, the active checkpoint at 674 comparisons has mean distortion 1.0976, close to passive GLHF's final 1.0986. The intervals apply to individual checkpoints and do not establish a uniform advantage along either curve. The implemented protocols differ in both pair selection and score normalization: active sampling uses $R=m-1$ in \eqref{eq:estimated-game}, whereas passive sampling uses $R=\binom m2$. These experiments compare the complete protocols on fixed sets of four and six alternatives; they do not measure the asymptotic dependence on $m$.

\FloatBarrier
\section{Related work}
\label{sec:related}

Social-choice approaches to alignment~\citep{conitzer2024} and pluralistic
alignment~\citep{sorensen2024} study how to accommodate diverse preferences.
We study a strong form of this goal: every prespecified group should receive
welfare comparable to its own best feasible fixed lottery, even when groups
overlap and have different optimal lotteries.

\paragraph{Population welfare and distortion.}
Distortion quantifies the welfare lost when we must optimize using only ordinal preferences
rather than cardinal utilities~\citep{procaccia2006,boutilier2015};
\citet{anshelevich2021} survey this literature. Under individual Bradley--Terry
comparisons, \citet{golz2025} establish the population benchmark used here.
Their Theorem~7 bounds NLHF's asymptotic distortion by $L/\ell_\beta$ for
every reference lottery and KL budget, relative to the best feasible lottery.
Their Theorem~3 proves a matching worst-case lower bound with one comparison
per user, using unconstrained instances with sufficiently many alternatives.
We apply their welfare conversion within each group and separately derive
finite-sample comparison guarantees and refinements for preference-aligned
groups. 

Recent work by \citet[Theorem~4.1]{oko2026} analyzes RLHF with
\emph{reward clipping}, which truncates the fitted rewards before policy
optimization. When both alternatives in each comparison are drawn
independently from the reference lottery, their analysis gives
$O(\beta)$ distortion for the same Bradley--Terry parameter $\beta\ge1$,
in the limit of infinitely many comparisons for reward fitting. This means
the learned lottery's welfare is at least a constant times $1/\beta$ of
the best feasible lottery's welfare. The bound matches the large-$\beta$
order of the optimal NLHF factor $L/\ell_\beta\sim\beta/2$, up to
constants, rather than establishing the same exact factor.
These results concern one lottery serving the whole population. Our paper
gives finite-sample guarantees simultaneously for arbitrary prespecified,
possibly intersecting groups of users.

\paragraph{Learning with overlapping-group guarantees.}
Early work on subgroup fairness emphasized that guarantees for coarse
demographic groups can mask large errors on structured subgroups. One line sought 
approximately equal positive-prediction rates or false-positive rates across
rich, overlapping group families~\citep{kearns2018}. Another (multicalibration) required
predicted probabilities to match outcome frequencies within each subgroup:
among members assigned a given probability, the outcome should occur at
approximately that rate~\citep{hebertjohnson2018}. 

Multigroup supervised learning asks for a single predictor whose loss on
each group is close to that of the group's own best predictor in a specified
class~\citep{rothblum2021}. \citet{blum2020} obtain analogous guarantees for
sequential prediction through \emph{sleeping experts}: a group's predictors
participate only on examples belonging to that group.
\citet{tosh2022} and \citet{globusharris2022} give simple, efficient algorithms
for obtaining these guarantees from independent training examples. \citet{acharya2024} give a simple, efficient reduction from groupwise to individual guarantees in the online setting. GLHF instantiates the architecture of their reduction with a learner over feasible
lotteries for each group.

\paragraph{Personalized and group-conditioned preference learning.}
Group Preference Optimization trains a preference predictor across several
groups. Given a small number of examples from a target group, including one
not seen during training, it predicts that group's preferences and uses
those predictions to guide model responses~\citep{zhao2024gpo}. Its focus is
adaptation from limited group-specific data, rather than a welfare guarantee
relative to each group's best feasible lottery.

Group Robust Preference Optimization (GRPO) conditions a policy on a group
label and minimizes the largest group-specific preference-training
loss~\citep{ramesh2024}. It uses Direct Preference Optimization (DPO) or
Identity Preference Optimization (IPO), which train policies directly from
preferred and rejected response pairs. Writing $L_k^{\mathrm{DPO}}(\pi)$ for
the average DPO training loss on group $k$, its DPO objective is
\[
    \min_{\pi}\max_{k\in[K]} L_k^{\mathrm{DPO}}(\pi).
\]
GRPO adaptively increases the weight on groups with larger losses. This
controls the largest group training loss, not each group's excess loss above
its own optimum: a group below the maximum can still be far from its own
best achievable loss.

GLHF gives simultaneous cardinal-welfare and distortion bounds relative to
\emph{each group's own} best feasible fixed lottery, under individual
Bradley--Terry comparisons. Thus the benchmark is group-specific, rather
than a single bound on the largest training loss; the two approaches also
measure performance differently, using welfare and preference-training loss,
respectively.

Other approaches exploit structure shared across users' rewards.
\citet{park2024} learn personalized reward models by sharing representations
or clustering users, with sample-complexity guarantees, and also study
preference aggregation. \citet{zhong2026} estimate stakeholder-specific
rewards from previously collected, stakeholder-labeled comparisons, then
choose one collective policy to optimize a welfare objective. Their analysis
assumes shared low-dimensional reward structure and sufficient data coverage
to evaluate candidate policies. Projection Optimization combines common
reward objectives using different group-specific weights, seeking consensus
or optimizing an objective aggregated across groups~\citep{xiong2025}.
GLHF's welfare bounds allow arbitrary utility variation within and across
groups under the individual Bradley--Terry model; they do not require a
shared parametric reward representation.

\paragraph{Other ways to represent or protect heterogeneous preferences.}
Distributional preference learning predicts distributions of reward scores
to capture disagreement arising from unobserved user characteristics or
evaluation criteria~\citep{siththaranjan2024}. \citet{halpern2025} learn a
collection of reward functions so that, for each response pair, the fraction
of functions preferring one response matches the fraction of annotators
preferring it, without annotator identities or predefined groups.
\citet{kim2026ppa} seek a shared lottery that represents groups with different
favorite alternatives in proportion to their population shares, while
limiting the benefit of misreporting preferences. These approaches represent
population diversity; GLHF uses observed memberships to give conditional
welfare guarantees.

Closer to welfare protection, \citet{latham2026} optimize target-population
welfare, penalized for deviation from a reference policy, subject to
specified minimum welfare levels for protected groups.
They learn from previously collected comparisons supplied by separate
population and group feedback sources, assuming linear Bradley--Terry
reward models. Their protected-group guarantees concern these absolute
welfare floors.
GLHF instead compares every named group's welfare with its own best feasible
fixed lottery, allowing heterogeneous utilities within each group.

\section*{Acknowledgements}
AR is supported in part by the Simons Collaboration on Algorithmic Fairness,
NSF grant IIS-2551608, and the NSF ENCoRE TRIPODS institute.

We used OpenAI's Codex and Anthropic's Claude to assist with mathematical
analysis, writing and revision, and the development of experimental code.
The authors take responsibility for the final manuscript, including all
proofs and empirical results.

\bibliographystyle{plainnat}
\bibliography{references}
\clearpage
\appendix
\section{Worst-case welfare lower bound}
\label{app:barrier}

\begin{proof}[Proof of Theorem~\ref{thm:barrier}]
We give the two-type construction of
\citet[Theorem~3 and Appendix~E.3]{golz2025}, then allow contextual
deployment. Fix an alternative $a^\circ$ and a small $h\in(0,1)$. Let
\[
s:=\sigma(\beta)-\tfrac12,
\qquad t_h:=\sigma(\beta h)-\tfrac12,
\qquad \theta_h:=\frac{t_h}{s+t_h}.
\]
With probability $\theta_h$, a user's utility is $1$ for $a^\circ$ and $0$
for all other alternatives. Otherwise, it is $0$ for $a^\circ$ and $h$
for all others. Every pairwise comparison is a fair coin: alternatives
other than $a^\circ$ have equal utility, and the margin of $a^\circ$
against any other alternative is
$2\theta_hs-2(1-\theta_h)t_h=0$.
The welfare of $a^\circ$ is $\theta_h$, while the welfare of every other
alternative, as a fraction of this, is
\[
r_h:=\frac{(1-\theta_h)h}{\theta_h}
=\frac{hs}{t_h}
\longrightarrow\frac{\ell_\beta}{L}
\quad\text{as }h\downarrow0.
\]
Strict concavity of $\sigma$ on the positive half-line gives $r_h<1$,
so $a^\circ$ is the welfare maximizer. Choose $h$ small enough and $m$
large enough that $r_h+(1-r_h)/m<c$.

Give all $m$ instances, one for each choice of $a^\circ$, the same context
distribution, independent of utilities. Each label is then a fair coin
conditional on the history, context, and queried pair. Hence the training
transcript and output policy have the same law in all instances, even with
adaptive pair selection. For a specified positive-mass group $k$, let
$\bar\pi_k:=\E[\Ek[\widehat\pi(X)]]$ average the output over fresh group
contexts, training data, and algorithmic randomness. Some alternative has
$\bar\pi_k(a^\circ)\le1/m$. In its corresponding instance, independence
of context and utility gives
\[
\frac{\E[W_k(\widehat\pi)]}{W_k^\star}
=\bar\pi_k(a^\circ)+(1-\bar\pi_k(a^\circ))r_h
\le r_h+\frac{1-r_h}{m}<c.
\]
The uniform reference and $\tau=\log m$ make every lottery feasible, so
$W_k^\star=\theta_h$. This proves the finite-sample claim. The finite
family of instances is fixed across sample sizes; for any sequence of
algorithms, at least one instance satisfies the same strict upper bound
at infinitely many sample sizes. Thus a limiting-inferior guarantee of
$c$ is impossible as well.
\end{proof}

\section{The comparison score}
\label{app:feedback}

This appendix verifies the unbiasedness, exact cancellation, and coordinate
bound used in Section~\ref{sec:algorithm}. Section~\ref{sec:score-moments}
derives the second-moment bound for both sampling protocols.

\begin{proposition}[Comparison-score properties]
\label{prop:estimator}
Under either sampling protocol, the score $\ghat_t$ in
\eqref{eq:estimated-game} satisfies \eqref{eq:score-interface} and
\eqref{eq:score-span}, with $R$
given by \eqref{eq:protocol-scale}.
\end{proposition}

\begin{proof}
Condition throughout on $(\F_{t-1},X_t)$, and write
$\pi=\pihat_t(X_t)$ and $\rho(c,d)=\rho_t(c,d)$. The lottery $\pi$ is
fixed. By Appendix~\ref{app:feasibility-invariant}, it has full support, so
$\rho(c,d)>0$ for every pair of distinct alternatives under either protocol.
Since $\rho(c,d)$ is the probability of the unordered pair $\{c,d\}$,
we also have $\rho(c,d)=\rho(d,c)$.
Label the observed winner $a$ and loser $b$, as in \eqref{eq:estimated-game}.
The two nonzero coordinates of $\ghat_t$ give the exact cancellation
\[
\pi^\top\ghat_t
=\frac{\pi(a)\pi(b)-\pi(b)\pi(a)}{R\rho(a,b)}=0.
\]

Fix a coordinate $c$. If the sampled pair is $\{c,d\}$, then
$\ghat_{t,c}=\pi(d)/(R\rho(c,d))$ when $c$ wins and
$\ghat_{t,c}=-\pi(d)/(R\rho(c,d))$ when $d$ wins. The pair is selected
independently of $U_t$ given the history and $X_t$. Hence, conditional on
sampling $\{c,d\}$, the probability that $c$ wins minus the probability
that $d$ wins is $M_{cd}(X_t)$, and
\[
\begin{aligned}
\E[\ghat_{t,c}]
&=\sum_{d\ne c}\rho(c,d)
  \frac{\pi(d)}{R\rho(c,d)}M_{cd}(X_t)\\
&=\frac1R\sum_{d\ne c}M_{cd}(X_t)\pi(d)
=\frac{(M(X_t)\pi)_c}{R}.
\end{aligned}
\]
The coordinate $\ghat_{t,c}$ is zero unless $c$ is in the sampled pair.
If the sampled pair is $\{c,d\}$, then \eqref{eq:pair-probabilities} gives
$|\ghat_{t,c}|=\pi(d)/(\pi(c)+\pi(d))$ under active sampling and
$|\ghat_{t,c}|=\pi(d)$ under passive sampling. Both are at most one.
The maximum and minimum score coordinates occur at the winner and loser,
so the score span is $(\pi(a)+\pi(b))/(R\rho(a,b))$. This equals one
under active sampling and is at most one under passive sampling, proving
\eqref{eq:score-span}.
Finally, \eqref{eq:score-second-moment} establishes
$\E[\ghat_{t,c}^2]\le1/R$ under the same conditioning.
\end{proof}

\section{KL projection}
\label{app:projection}

An inner learner's exponentiated step in \eqref{eq:inner-update} may leave
$\B_\tau$. Here we show how to project it back with a one-dimensional search
and prove the KL projection inequality used in the regret analysis.

The map $q\mapsto\KL(q\|\piref)$ is continuous and convex on the
simplex, so $\B_\tau$ is compact, convex, and contains $\piref$.
Thus mixtures of feasible lotteries remain feasible. The radius cap in
Section~\ref{sec:model} follows from
\[
\KL(q\|\piref)
=\sum_aq(a)\log q(a)-\sum_aq(a)\log\piref(a)
\le\log\frac1{\min_a\piref(a)},
\]
with equality at a point mass on an alternative of minimum reference
probability. The KL projection inequality is standard
\citep[Theorem~3.1]{csiszarshields2004}. We give a direct proof and a
one-dimensional way to find the projection.

For a full-support $v\in\simplex$, consider the geometric path
\begin{equation}
q_s(a):=\frac{\piref(a)^{1-s}v(a)^s}{Z_s},
\qquad
Z_s:=\sum_b\piref(b)^{1-s}v(b)^s,
\qquad s\in[0,1].
\label{eq:projection-path}
\end{equation}
At $s=0$ this is $\piref$, at $s=1$ it is $v$, and every $q_s$ has full support.

\begin{proposition}[Exact projection]
\label{prop:projection}
If $v\in\B_\tau$, its KL projection onto $\B_\tau$ is $v$. If $v\notin\B_\tau$ and $\tau>0$, there is a unique $s_\star\in(0,1)$ satisfying $\KL(q_{s_\star}\|\piref)=\tau$, and
\begin{equation}
\argmin_{q\in\B_\tau}\KL(q\|v)=q_{s_\star}.
\label{eq:exact-projection}
\end{equation}
At $\tau=0$, the projection is $\piref$.
Writing $p$ for this projection, every feasible $q$ satisfies
\begin{equation}
\KL(q\|p)+\KL(p\|v)\le\KL(q\|v).
\label{eq:pythagoras}
\end{equation}
\end{proposition}

\begin{proof}
If $v\in\B_\tau$, its projection is $v$; if $\tau=0$, the feasible set is
$\{\piref\}$. In either case \eqref{eq:pythagoras} holds with equality.
Otherwise, write
$h_a=\log(v(a)/\piref(a))$ and
$A(s)=\log\sum_a\piref(a)e^{sh_a}$.  Then
$q_s(a)=\piref(a)e^{sh_a-A(s)}$.  Direct differentiation gives
$A'(s)=\E_{q_s}[h]$ and $A''(s)=\Var_{q_s}(h)$.  Moreover
\[
\KL(q_s\|\piref)=sA'(s)-A(s),
\qquad
\frac{\mathrm d}{\mathrm ds}\KL(q_s\|\piref)
=s\Var_{a\sim q_s}(h_a).
\]
Since $v\ne\piref$, $h$ is nonconstant and $\Var_{q_s}(h)>0$ by full
support. Thus the derivative is positive for $s>0$, and
$\KL(q_s\|\piref)$ increases continuously from zero to
$\KL(v\|\piref)>\tau$, giving a unique $s_\star\in(0,1)$.

To identify this boundary point as the projection, use
$\log q_s(a)=(1-s)\log\piref(a)+s\log v(a)-\log Z_s$ from
\eqref{eq:projection-path}. Expanding the KL terms gives, for any
$q\in\simplex$ and $s\in[0,1]$,
\begin{equation}
\KL(q\|q_s)
=s\big[\KL(q\|v)-\KL(q_s\|v)\big]
+(1-s)\big[\KL(q\|\piref)-\KL(q_s\|\piref)\big].
\label{eq:path-identity}
\end{equation}
Set $p=q_{s_\star}$. For feasible $q$, the second bracket is nonpositive
because $\KL(p\|\piref)=\tau$. Since $s_\star<1$,
\[
\KL(q\|v)-\KL(p\|v)
\ge\frac1{s_\star}\KL(q\|p)
\ge\KL(q\|p).
\]
For $q\ne p$, $\KL(q\|p)>0$, so the first inequality proves uniqueness;
the second gives \eqref{eq:pythagoras}.
\end{proof}

\section{Concentration for groupwise scores}
\label{app:concentration}

The population conversion in \eqref{eq:population-conversion} requires
bounds on $V_{T,k}$ and on $q^\top S_k$ for every fixed feasible lottery
$q$. We work at the fixed horizon $T$, with
$\Et[\cdot]=\E[\cdot\mid\F_{t-1}]$ as in
Section~\ref{sec:analysis}.

\begin{lemma}[A bounded nonnegative sum]
\label{lem:nonnegative}
Let $b>0$ and let $Z_t\in[0,b]$ depend only on the history through round
$t$. Suppose
$\sum_{t=1}^T\Et[Z_t]\le u$ almost surely for a deterministic $u\ge0$.
Then, for every $\gamma\in(0,1)$,
\[
\Pr\left(\sum_{t=1}^T Z_t>2u+b\log\frac1\gamma\right)\le\gamma.
\]
\end{lemma}
\begin{proof}
Put $\mu_t=\Et[Z_t]$. The chord bound
$e^z\le1+(e-1)z$ for $z\in[0,1]$ gives
$\Et[e^{Z_t/b}]\le1+(e-1)\mu_t/b\le e^{(e-1)\mu_t/b}$.
For each running round $s=0,\ldots,T$, define
\[
\mathcal L_s:=\exp\!\left(\frac1b\sum_{t=1}^s[Z_t-(e-1)\mu_t]\right),
\qquad \mathcal L_0=1.
\]
The process is positive and adapted to $(\F_s)$. Since $\mu_s$ is
$\F_{s-1}$-measurable, the preceding bound implies
\[
\E[\mathcal L_s\mid\F_{s-1}]
=\mathcal L_{s-1}e^{-(e-1)\mu_s/b}\E[e^{Z_s/b}\mid\F_{s-1}]
\le \mathcal L_{s-1}.
\]
Thus $(\mathcal L_s)_{s=0}^T$ is a nonnegative supermartingale with respect to
$(\F_s)$, and $\E[\mathcal L_T]\le1$. Since $e-1<2$ and
$\sum_t\mu_t\le u$,
$e^{(\sum_t Z_t-2u)/b}\le \mathcal L_T$, so
$\E[e^{(\sum_t Z_t-2u)/b}]\le1$.
Markov's inequality gives the claim.
\end{proof}

The next lemma combines the Bernstein exponential argument for bounded
martingales \citep{freedman1975} with the finite KL change-of-measure
inequality \citep[Lemma~6]{seldin2012}. We give the fixed-horizon form
and constants used here.

\begin{lemma}[One martingale event for every feasible lottery]
\label{lem:kl-uniform}
Let $Z_t\in\R^m$ be martingale differences with
$|Z_{t,a}|\le b$ and
$\sum_{t=1}^T\Et[Z_{t,a}^2]\le v$ almost surely for every coordinate
$a$, where $b>0$ and $v\ge0$ are deterministic. Let $\piref$ be the
full-support reference, fixed before the sequence is observed. For $\tau\ge0$ and
$\gamma\in(0,1)$, with probability at least $1-\gamma$,
\begin{equation}
\sup_{q\in\B_\tau}q^\top\sum_{t=1}^T Z_t
\le 2\sqrt{v\left(\tau+\log\frac1\gamma\right)}
+\frac43b\left(\tau+\log\frac1\gamma\right).
\label{eq:kl-uniform}
\end{equation}
\end{lemma}
\begin{proof}
If $v=0$, all increments vanish almost surely. Otherwise put
$S=\sum_tZ_t$, $h=\tau+\log(1/\gamma)$, and choose the deterministic
$\zeta=\min\{3/(2b),\sqrt{h/v}\}$. For
$\sigma_{t,a}^2=\Et[Z_{t,a}^2]$, the bounded-martingale
Bernstein moment bound gives
\[
\Et[e^{\zeta Z_{t,a}}]
\le\exp\!\left(\frac{\zeta^2\sigma_{t,a}^2}
{2(1-b\zeta/3)}\right)
\le e^{\zeta^2\sigma_{t,a}^2}.
\]
The first inequality follows by expanding the exponential, using
$\Et[Z_{t,a}]=0$, $|Z_{t,a}|^j\le b^{j-2}Z_{t,a}^2$, and
$j!\ge2\cdot3^{j-2}$ for $j\ge2$.
The second inequality uses $b\zeta\le3/2$.
For each $a$,
$\mathcal L_{s,a}=\exp\bigl(\zeta\sum_{t=1}^sZ_{t,a}
-\zeta^2\sum_{t=1}^s\sigma_{t,a}^2\bigr)$ is therefore a nonnegative
supermartingale with $\mathcal L_{0,a}=1$. Since
$\sum_t\sigma_{t,a}^2\le v$,
$\E[e^{\zeta S_a-\zeta^2v}]\le\E[\mathcal L_{T,a}]\le1$ for each $a$. Hence the fixed
mixture $\mathcal Z_\zeta=\sum_a\piref(a)e^{\zeta S_a}$ obeys
$\E[e^{-\zeta^2v}\mathcal Z_\zeta]\le1$. By Markov's inequality, with probability at least
$1-\gamma$,
\[
\log \mathcal Z_\zeta\le\zeta^2v+\log(1/\gamma).
\]
This is one event for the entire mixture; no union bound over the $m$
coordinates is needed.
On this event, define $r_\zeta(a)=\piref(a)e^{\zeta S_a}/\mathcal Z_\zeta$. For every
$q\in\simplex$,
\[
0\le\KL(q\|r_\zeta)
=\KL(q\|\piref)-\zeta q^\top S+\log \mathcal Z_\zeta.
\]
Thus, simultaneously for every $q\in\B_\tau$,
\[
q^\top S\le\frac{\KL(q\|\piref)+\log \mathcal Z_\zeta}{\zeta}
\le\frac{h}{\zeta}+\zeta v.
\]
If $\zeta=\sqrt{h/v}$, the last expression is $2\sqrt{vh}$.
Otherwise $\zeta=3/(2b)$, so $v\le4b^2h/9$ and the bound is at most
$2bh/3+3v/(2b)\le4bh/3$. These two cases imply
\eqref{eq:kl-uniform}.
\end{proof}

\subsection{The simultaneous concentration event}
\label{app:concentration-proof}

Recall the realized second moment $V_{T,k}$ from
\eqref{eq:second-moments} and the sampling-error martingale
$S_k$ from \eqref{eq:martingale}. Define the deterministic analysis bound
\begin{equation}
V_k^{\mathrm{ub}}:=\frac{2Tp_k}{R}+\log\frac{2K}{\delta}.
\label{eq:second-moment-upper}
\end{equation}
The algorithm does not use this bound, the group masses, or the confidence
level. Only the proof uses them to control the realized second moments and
sampling error.

\begin{lemma}[One event for all groups and comparators]
\label{lem:concentration}
With probability at least $1-\delta$, simultaneously for every group
$k\in[K]$,
\begin{align}
V_{T,k}&\le V_k^{\mathrm{ub}},\notag\\
\sup_{q\in\B_\tau}q^\top S_k
&\le2\sqrt{\frac{Tp_k}{R}\left(\tau+\log\frac{2K}{\delta}\right)}
+\frac83\left(\tau+\log\frac{2K}{\delta}\right).
\label{eq:concentration}
\end{align}
\end{lemma}
\begin{proof}
Fix $k$. Set
$v_{t,k}=G_k(X_t)\sum_a\pi_{t,k}(a)\ghat_{t,a}^2$, so that
$V_{T,k}=\sum_t v_{t,k}$.
The score bound $|\ghat_{t,a}|\le1$ and the fact that $\pi_{t,k}$ is a
lottery give $v_{t,k}\in[0,1]$.
The lottery $\pi_{t,k}$ is fixed by $\F_{t-1}$. Conditioning first on
$(\F_{t-1},X_t)$ and using \eqref{eq:score-interface} gives
\begin{align*}
\Et[v_{t,k}]&=\Et\!\left[G_k(X_t)\sum_a\pi_{t,k}(a)
 \E[\ghat_{t,a}^2\mid\F_{t-1},X_t]\right]\\
&\le\frac1R\Et\!\left[G_k(X_t)\sum_a\pi_{t,k}(a)\right]\\
&=\frac1R\Et[G_k(X_t)]=\frac{p_k}{R}.
\end{align*}
The last equality uses independence of the fresh context from the
history.
Lemma~\ref{lem:nonnegative} with $b=1$, $u=Tp_k/R$, and
$\gamma=\delta/(2K)$ therefore gives
$V_{T,k}\le2Tp_k/R+\log(2K/\delta)=V_k^{\mathrm{ub}}$
except with probability $\delta/(2K)$.

For the sampling error, set $Y_{t,k}=G_k(X_t)\ghat_t$.
The martingale differences $Z_t=\Et[Y_{t,k}]-Y_{t,k}$ sum to
$S_k$ and have $|Z_{t,a}|\le2$, since both $Y_{t,k,a}$ and its conditional mean lie
in $[-1,1]$. For each coordinate, the same conditioning argument gives
\begin{align*}
\Et[Z_{t,a}^2]&=\Var(Y_{t,k,a}\mid\F_{t-1})\\
&\le\Et[Y_{t,k,a}^2]\\
&=\Et\!\left[G_k(X_t)
 \E[\ghat_{t,a}^2\mid\F_{t-1},X_t]\right]\\
&\le\frac1R\Et[G_k(X_t)]=\frac{p_k}{R}.
\end{align*}
Thus $\sum_t\Et[Z_{t,a}^2]\le Tp_k/R$ for every coordinate $a$.
Lemma~\ref{lem:kl-uniform} with $b=2$, $v=Tp_k/R$, and
$\gamma=\delta/(2K)$ gives the
second bound except with probability $\delta/(2K)$.
A union bound over these two events for each of the $K$ groups proves
the claim.
\end{proof}

\section{Proof details for the comparison guarantee}
\label{app:proof-details}

The outer proof tracks a scaled log weight. The inner proof factors the
update into three steps and tracks a centered, scaled KL divergence.
Both quantities telescope despite changing rates. Throughout, $\tau>0$.

\subsection{Positivity and pointwise feasibility}
\label{app:feasibility-invariant}

Initially, the outer weights are positive and each specialist lottery is the
full-support feasible reference. Suppose these properties hold at round $t$.
Because the groups cover every context, deployment is a well-defined,
full-support convex combination of feasible lotteries. The queried-pair
probabilities are positive, and the score has coordinates in $[-1,1]$.
The tentative outer factor $1+\alpha_{t,k}r_{t,k}$ lies in $[1/2,3/2]$.
The subsequent geometric mixture with the positive initial weight preserves
positivity. Both inner rates are positive, and the geometric mixture of the
current lottery and reference, followed by exponentiation, has full support.
The KL projection in Proposition~\ref{prop:projection} preserves full support
and restores feasibility. Inactive states are unchanged, completing the
induction.

\subsection{The outer log-weight bound}
\label{app:outer-bound}

The tentative weights in \eqref{eq:mass-cancellation} conserve their total
mass. The rate-change correction can increase that mass, but its total
increase is logarithmic in the horizon.

Write $W_t=\sum_k w_{t,k}$ and
$\gamma_{t,k}=\alpha_{t+1,k}/\alpha_{t,k}\in(0,1]$.
For an active group, let
$\widetilde w_{t+1,k}=w_{t,k}(1+\alpha_{t,k}r_{t,k})$.
The outer update is
$w_{t+1,k}=K^{-1}(K\widetilde w_{t+1,k})^{\gamma_{t,k}}$.
Since $x^\gamma\le x+1-\gamma$ for $x\ge0$ and $0<\gamma\le1$,
\[
W_{t+1}\le W_t+\frac1K\sum_{k:G_k(X_t)=1}(1-\gamma_{t,k}).
\]
The clipped square-root rate in \eqref{eq:rates} satisfies
\[
\gamma_{t,k}\ge
\sqrt{\frac{1+Q_{t,k}}{1+Q_{t+1,k}}},
\qquad
1-\gamma_{t,k}\le
\frac12\log\frac{1+Q_{t+1,k}}{1+Q_{t,k}}.
\]
For the second inequality, use $1-e^{-z}\le z$ for $z\ge0$.
Summing over rounds, using $W_1=1$ and $Q_{1,k}=0$, gives
\[
W_{T+1}\le1+\frac1{2K}\sum_k\log(1+Q_{T+1,k})
\le1+\frac12\log(T+1),
\]
because $|r_{t,k}|\le1$ implies $Q_{T+1,k}\le T$.
Consequently,
$\log(Kw_{T+1,k})\le\Lambda_T$.

To bound group $k$'s accumulated score, divide its log-weight update
by the next learning rate:
\[
\frac{\log(Kw_{t+1,k})}{\alpha_{t+1,k}}
=
\frac{\log(Kw_{t,k})}{\alpha_{t,k}}
+\frac{\log(1+\alpha_{t,k}r_{t,k})}{\alpha_{t,k}}.
\]
The equality holds on active rounds; on inactive rounds both states and
rates are unchanged. Since $|\alpha_{t,k}r_{t,k}|\le1/2$,
$\log(1+z)\ge z-z^2$ yields
\[
\sum_{t:G_k(X_t)=1}r_{t,k}
\le\frac{\log(Kw_{T+1,k})}{\alpha_{T+1,k}}
+\sum_{t:G_k(X_t)=1}\alpha_{t,k}r_{t,k}^2
\le\frac{\Lambda_T}{\alpha_{T+1,k}}
+\sum_{t:G_k(X_t)=1}\alpha_{t,k}r_{t,k}^2.
\]
The initial scaled log weight is zero because $w_{1,k}=1/K$.
This proves \eqref{eq:outer-regret}.

\subsection{The inner one-step inequality}
\label{app:inner-step}

Fix a group $k$ and $q\in\B_\tau$. Let
$D_t=\KL(q\|\pi_{t,k})$ and $d=\KL(q\|\piref)\le\tau$.
On an active round abbreviate
$\eta=\eta_{t,k}$, $\eta'=\eta_{t+1,k}$, and
$\gamma=\eta'/\eta\in(0,1]$.
Factor \eqref{eq:inner-update} into multiplicative weights, geometric
interpolation toward $\piref$, and KL projection:
\[
v_t(a)\propto\pi_{t,k}(a)e^{\eta\ghat_{t,a}},
\qquad
z_{t,k}(a)=\frac{v_t(a)^\gamma\piref(a)^{1-\gamma}}{Z_t},
\qquad
Z_t:=\sum_a v_t(a)^\gamma\piref(a)^{1-\gamma}.
\]
Here $v_t$ is normalized, $\gamma\eta=\eta'$, and the projection of
$z_{t,k}$ is $\pi_{t+1,k}$.

Since $\eta\le1$ and $|\ghat_{t,a}-r_{t,k}|\le1$ by \eqref{eq:score-span},
centering the score and using $e^z\le1+z+z^2$ for $|z|\le1$
and $\log u\le u-1$ give
\begin{align*}
\KL(q\|v_t)
&=D_t-\eta(q-\pi_{t,k})^\top\ghat_t+
\log\sum_a\pi_{t,k}(a)e^{\eta(\ghat_{t,a}-r_{t,k})}
\\
&\le D_t-\eta(q-\pi_{t,k})^\top\ghat_t+\eta^2c_{t,k}.
\end{align*}
The weighted arithmetic--geometric mean inequality gives
$Z_t\le\sum_a[\gamma v_t(a)+(1-\gamma)\piref(a)]=1$.
Together with the projection inequality in Proposition~\ref{prop:projection},
this yields
\begin{align*}
D_{t+1}
&\le\KL(q\|z_{t,k})
=\gamma\KL(q\|v_t)+(1-\gamma)d+\log Z_t\\
&\le\gamma D_t+(1-\gamma)d
-\eta'(q-\pi_{t,k})^\top\ghat_t+\eta'\eta c_{t,k}.
\end{align*}
Thus the centered potential $\Psi_{t,k}(q):=(D_t-d)/\eta_{t,k}$ satisfies
\begin{equation}
(q-\pi_{t,k})^\top\ghat_t
\le\Psi_{t,k}(q)-\Psi_{t+1,k}(q)+\eta_{t,k}c_{t,k}.
\label{eq:inner-step}
\end{equation}
Inactive rounds leave $\Psi_{t,k}$ unchanged. Since $\pi_{1,k}=\piref$,
$\Psi_{1,k}=0$, while $\Psi_{T+1,k}\ge-d/\eta_{T+1,k}$.
Summing gives \eqref{eq:inner-regret}:
\[
\sum_{t:G_k(X_t)=1}(q-\pi_{t,k})^\top\ghat_t
\le-\Psi_{T+1,k}(q)
+\sum_{t:G_k(X_t)=1}\eta_{t,k}c_{t,k}
\le\frac{\tau}{\eta_{T+1,k}}
+\sum_{t:G_k(X_t)=1}\eta_{t,k}c_{t,k}.
\]

\subsection{Summing the adaptive-rate terms}
\label{app:rate-calculation}

We first bound a sum that occurs in both updates. Let $b>0$,
$\Delta_t\in[0,b]$, and $S_t=\sum_{s<t}\Delta_s$. Then
\begin{equation}
\sum_{t=1}^T\frac{\Delta_t}{\sqrt{b+S_t}}
\le(1+\sqrt2)\bigl(\sqrt{b+S_{T+1}}-\sqrt b\bigr).
\label{eq:rate-bound}
\end{equation}
Indeed, each summand equals
\[
\bigl(\sqrt{b+S_{t+1}}-\sqrt{b+S_t}\bigr)
\left(1+\sqrt{\frac{b+S_{t+1}}{b+S_t}}\right).
\]
The second factor is at most $1+\sqrt2$ because
$\Delta_t\le b\le b+S_t$; summing the first factors telescopes.

For the remainder of this subsection, write
$Q=Q_{T+1,k}$ and $C=C_{T+1,k}$.
Apply \eqref{eq:rate-bound} with $b=1$ and active increments $r_{t,k}^2$.
The outer rate is at most
$\sqrt{\Lambda_T/(2(1+Q_{t,k}))}$, while
\[
\frac{\Lambda_T}{\alpha_{T+1,k}}
=\max\{2\Lambda_T,\sqrt{2\Lambda_T(1+Q)}\}.
\]
It follows that the outer regret is at most
\[
2\Lambda_T+
\left(\sqrt2+\frac{1+\sqrt2}{\sqrt2}\right)
\sqrt{\Lambda_T(1+Q)}
\le2\Lambda_T+4\sqrt{\Lambda_T(1+Q)}.
\]
For the inner update, a random variable with range of length at most one
has variance at most $1/4$: its variance is bounded by its expected squared
distance from the midpoint of that range. Hence $c_{t,k}\le1/4$.
Apply \eqref{eq:rate-bound} with $b=1/4$ and active increments $c_{t,k}$,
and use
\[
\frac{\tau}{\eta_{T+1,k}}
=\max\{\tau,\sqrt{2\tau(1/4+C)}\}.
\]
The same calculation bounds the inner regret by
$\tau+4\sqrt{\tau(1/4+C)}$.

The identity $r_{t,k}^2+c_{t,k}=\sum_a\pi_{t,k}(a)\ghat_{t,a}^2$
gives $Q+C=V_{T,k}$. Combining the two bounds with Cauchy--Schwarz and
the decomposition \eqref{eq:decomposition} proves
\[
A_{T,k}(q)
\le4\sqrt{(\tau+\Lambda_T)(V_{T,k}+5/4)}
+2(\tau+\Lambda_T),
\]
as claimed in Lemma~\ref{lem:regret}.

\subsection{Detailed certificate calculation}
\label{app:certificate-proof}

On the event of Lemma~\ref{lem:concentration}, substitute the preceding
deterministic bound and both concentration bounds into
\eqref{eq:population-conversion}. For every $k$ and $q\in\B_\tau$,
\begin{align}
Tp_k\Ek[q^\top M(X)\piGLHF(X)]
&\le R\left[4\sqrt{(\tau+\Lambda_T)(V_k^{\mathrm{ub}}+5/4)}
+2(\tau+\Lambda_T)\right]\notag\\
&\quad+2\sqrt{RTp_k\left(\tau+\log\frac{2K}{\delta}\right)}
+\frac83R\left(\tau+\log\frac{2K}{\delta}\right).
\label{eq:modular-global}
\end{align}
To simplify the constants, write $A=\tau+\Lambda_T$,
$s=\log(2K/\delta)$, and $u=Tp_k/R$. Thus $H=A+s$ and
$V_k^{\mathrm{ub}}=2u+s$.
Divide the right side of \eqref{eq:modular-global} by $R$.
The inequality $\sqrt{x+y}\le\sqrt x+\sqrt y$ bounds it by
\[
4\sqrt{2Au}+2\sqrt{(\tau+s)u}
+4\sqrt{A(s+5/4)}+2A+\frac83(\tau+s).
\]
The first two terms are at most $(4\sqrt2+2)\sqrt{Hu}<8\sqrt{Hu}$.
For the others, AM--GM gives
$4\sqrt{A(s+5/4)}\le2A+2s+5/2$, and
\[
4A+2s+\frac52+\frac83(\tau+s)
\le\frac{20}{3}H<8H.
\]
Indeed, the difference in the first inequality is
$(8/3)\Lambda_T+2s-5/2>0$, since $\Lambda_T\ge1$ and $s>0$.
Restoring the factor $R$ and dividing by $Tp_k$ therefore gives
\[
\Ek[q^\top M(X)\piGLHF(X)]
\le8\sqrt{\frac{RH}{Tp_k}}+8\frac{RH}{Tp_k}
=\varepsilon_{T,k}.
\]
The event is uniform in $q$, so taking the supremum proves
\eqref{eq:main-certificate}.

The same argument includes $K=1$. In that case coverage makes the sole
group universal, score cancellation gives $r_{t,1}=0$, and its outer weight
remains one. At $\tau=0$, the algorithm returns $\piref$ and the feasible
set contains only that lottery, so antisymmetry gives an exact zero
comparison certificate.

\section{From group distortion to population distortion}
\label{app:welfare}
\label{app:population-transfer}

Write $W(\pi):=\E[\pi(X)^\top U]$ for population
welfare and $W^\star:=\max_{q\in\B_\tau}W(q)$ for its best fixed feasible
lottery benchmark. When $W(\pi)>0$, define population distortion by
$\Dist_{\mathrm{pop}}(\pi):=W^\star/W(\pi)$.

Group distortion bounds need not hold at the same factor for the population
when groups overlap, even if they cover every context, none is universal, and
every group has zero residual loss. For example,
take three equally likely contexts with identical utilities $(1,0)$ for
alternatives $A,B$, a uniform reference, and $\tau=\log2$, so both pure
actions are feasible. Let group $1$ contain contexts $\{1,2\}$ and group $2$
contain contexts $\{2,3\}$. A policy
serving $A$ only at context $2$ and $B$ elsewhere has
\begin{align*}
W_1(\pi)=W_2(\pi)&=\tfrac12,
& W(\pi)&=\tfrac13,\\
\Dist_1(\pi)=\Dist_2(\pi)&=2,
& \Dist_{\mathrm{pop}}(\pi)&=3.
\end{align*}
Every benchmark welfare is one and both residual losses are zero: the
well-served intersection contributes to both group averages. Thus a group
distortion bound of two does not imply the same bound for the population.

Small groupwise shortfalls from individual-best welfare do transfer,
even with overlap. Let $W_{\mathrm{ind}}:=\E[\max_aU(a)]$ and let
$I\subseteq[K]$ be any subcollection whose groups cover the population,
meaning $\Pr(\sum_{k\in I}G_k(X)\ge1)=1$. Under
Assumption~\ref{ass:bt}, if $\Expl_k(\pi)\le\varepsilon_k$ for $k\in I$,
then Theorem~\ref{prop:homogeneous} gives
\begin{equation}
W_{\mathrm{ind}}-W(\pi)
\le\sum_{k\in I}p_k\left[
\frac{\varepsilon_k}{2\beta\ell_\beta}
+\frac{L}{\ell_\beta}\mathrm{Res}_k\right].
\label{eq:population-shortfall}
\end{equation}
In particular, GLHF satisfies this bound with
$\varepsilon_k=\varepsilon_{T,k}$ on the success event of
Theorem~\ref{thm:certificate}. One may choose any covering subcollection
to keep the sum small.

To prove the bound, define the nonnegative individual shortfall
$Z:=\max_aU(a)-\pi(X)^\top U$. Coverage and
Theorem~\ref{prop:homogeneous} imply
\begin{align*}
0\le W_{\mathrm{ind}}-W(\pi)
&=\E[Z]\\
&\le\sum_{k\in I}p_k\Ek[Z]\\
&=\sum_{k\in I}p_k\bigl(\mathrm{Res}_k+W_k^\star-W_k(\pi)\bigr)\\
&\le\sum_{k\in I}p_k\left[
\frac{\varepsilon_k}{2\beta\ell_\beta}
+\frac{L}{\ell_\beta}\mathrm{Res}_k\right].
\end{align*}
Since $W_{\mathrm{ind}}\ge W^\star$, the same right-hand side also bounds
the population loss $W^\star-W(\pi)$. When both errors and residuals
vanish on a covering subcollection, \eqref{eq:population-shortfall} gives
$W(\pi)=W_{\mathrm{ind}}\ge W^\star$, and population distortion is at
most one whenever $W^\star>0$.
More generally, this exact conclusion requires only zero residuals and
$W_k(\pi)\ge W_k^\star$ on a covering subcollection: the penultimate line
of the same shortfall argument is then zero. Thus group distortion at most
one transfers when the groups have zero residual loss, even if they overlap.

For disjoint groups that cover the population, any common distortion bound
transfers without the zero-residual condition. If, for some $D>0$,
$W_k(\pi)\ge W_k^\star/D$ for every group and
$q^\star\in\argmax_{q\in\B_\tau}W(q)$, then
\[
W(\pi)=\sum_kp_kW_k(\pi)
\ge\frac1D\sum_kp_kW_k^\star
\ge\frac1D\sum_kp_kW_k(q^\star)
=\frac{W^\star}{D}.
\]

\section{Experimental details and additional results}
\label{app:empirical}

Our primary comparison is multigroup GLHF against NLHF, which serves everyone the same lottery. Both studies use ratings to generate training comparisons and measure held-out welfare. After describing the data and evaluation procedure, we report detailed method comparisons and sensitivity checks. Sections~\ref{sec:empirical-noise} and~\ref{sec:empirical-trajectories} give the comparison-noise and active/passive results. Implementation and collection details appear last.

\subsection{Data and group selection}
\label{app:empirical-data}

\paragraph{Coffee ratings.}
The Great American Coffee Taste Test~\citep{coffee2024} contains ratings of four coffees on a 1--5 scale. We retain the 3,744 respondents who rated all four and set $U=(r-1)/4$, treating equal rating increments as equal utility increments. We reserve a random discovery sample of 375 respondents for choosing groups, leaving 3,369 for analysis.

\paragraph{Synthetic language ratings.}
The completion-rating study uses six fixed answers to \emph{Why is the sky blue?}: four correct answers vary in technical detail and writing style, and two are incorrect controls. We construct sixteen personas using binary characteristics: age group, education, technical background, and preferred writing style. For each persona and candidate completion, 104 independent instances of Claude Haiku 4.5 (\texttt{claude-haiku-4-5-20251001}) were asked to assume the persona and provide a rating of the completion on a 1--10 scale. This gives 9,984 completion ratings, which we rescaled as $U=(r-1)/9$. Within each persona, we independently shuffled the ratings for each answer, then combined one rating for each of the six answers into one pseudo-respondent's utility vector. This produced 104 pseudo-respondents per persona and 1,664 in total. Each collected rating was used once. We reserve 167 pseudo-respondents for group selection (10 or 11 per persona), leaving 1,497 for analysis. Appendix~\ref{app:persona-materials} specifies the synthetic data collection procedure.

\paragraph{Group definitions.}
Each study includes the universal group (everyone in the analysis sample) and five subgroups that may overlap. Tables~\ref{tab:coffee-groups} and~\ref{tab:persona-groups} give definitions and analysis-sample counts. Coffee respondents who did not report a demographic attribute remain in the universal group and are not included in subgroups defined by that attribute.

\begin{table}[htbp]
\centering
\small
\begin{tabular}{lrp{0.53\linewidth}}
\toprule
Group & $n$ & Definition \\
\midrule
Universal & 3,369 & Every analysis respondent. \\
Age 45+ & 483 & Reported age 45--54, 55--64, or over 65. \\
Female & 759 & Reported gender is female. \\
Expertise 1--4 & 815 & Self-rated coffee expertise is 1, 2, 3, or 4. \\
Children 1+ & 767 & Reported one, two, three, or more than three children. \\
Works from home & 1,424 & Selected ``I primarily work from home.'' \\
\bottomrule
\end{tabular}
\caption{Groups and membership counts in the coffee-rating analysis sample.}
\label{tab:coffee-groups}
\end{table}

\begin{table}[htbp]
\centering
\small
\begin{tabular}{lrp{0.53\linewidth}}
\toprule
Group & $n$ & Definition \\
\midrule
Universal & 1,497 & Every analysis pseudo-respondent. \\
Not technical & 750 & Persona specifies a nontechnical background. \\
Expressive wording & 747 & Persona prefers expressive writing. \\
Direct wording & 750 & Persona prefers direct writing. \\
Very technical & 747 & Persona specifies a very technical background. \\
University education & 748 & Persona specifies university education. \\
\bottomrule
\end{tabular}
\caption{Groups and analysis-sample counts in the completion-rating study.}
\label{tab:persona-groups}
\end{table}

\paragraph{Selecting groups for the coffee-rating study.}
We chose the five subgroups in Table~\ref{tab:coffee-groups} by judgment, using preference correlations in the discovery sample to guide selection. For each of the six coffee pairs, we computed the fraction preferring the first coffee minus the fraction preferring the second, with ties contributing zero. We calculated these six margins within each candidate group and within the full discovery sample, then compared the two vectors using Pearson correlation. The selected groups included three with negative correlations (Age 45+, Female, and Expertise 1--4) and two with positive correlations (Children 1+ and Works from home). We fixed these groups before the reported training and evaluation runs. We had previously explored the Coffee data, before reserving the discovery sample and finalizing these groups; this study is therefore exploratory, and the split does not remove that earlier exposure.

\paragraph{Selecting groups for the completion-rating study.}
We considered eight candidate subgroups, corresponding to the two levels of each persona attribute. For each of the 15 answer pairs, we calculated the average rescaled rating of the first answer minus the average rescaled rating of the second. We calculated these differences within each candidate group and within the full discovery sample, then compared them using Pearson correlation. Lower correlations indicate that the group's preferences across answer pairs are less aligned with those of the full discovery sample.

Candidate groups must contain at least 30 discovery rows and have a defined correlation. We select the four lowest-correlation candidates and the lower median of the remaining candidates ordered by correlation. The fifth group adds a more typical candidate to the four least aligned groups. Table~\ref{tab:persona-discovery} reports the selection, run once on the discovery sample. We fixed these groups before training and evaluation.

\begin{table}[htbp]
\centering
\small
\begin{tabular}{lrc@{\qquad}l}
\toprule
Persona attribute & Discovery $n$ & Rating-difference correlation & Selected \\
\midrule
Not technical & 82 & $-0.115$ & yes \\
Expressive wording & 85 & $0.511$ & yes \\
Direct wording & 82 & $0.643$ & yes \\
Very technical & 85 & $0.777$ & yes \\
High school or less & 83 & $0.983$ & no \\
University education & 84 & $0.988$ & yes \\
Age 22--32 & 82 & $0.996$ & no \\
Age 62--72 & 85 & $0.996$ & no \\
\bottomrule
\end{tabular}
\caption{Group selection in the completion-rating study on the 167 reserved discovery rows. Correlations compare each candidate's 15 mean-rating differences with those of the full discovery sample. The universal group was included separately.}
\label{tab:persona-discovery}
\end{table}

\subsection{Experimental design and evaluation}
\label{app:empirical-training}

\paragraph{Methods and main setting.}
Table~\ref{tab:empirical-methods} summarizes the four methods. Multigroup GLHF uses memberships both to update group specialists and to choose the lottery served to a user. Batch NLHF is the primary population-wide comparator. Pooled GLHF is the same implementation as multigroup GLHF with only the universal group. The uniform reference shows how learning compares to serving the starting lottery.

\begin{table}[htbp]
\centering
\small
\begin{tabular}{@{}p{0.22\linewidth}p{0.44\linewidth}p{0.25\linewidth}@{}}
\toprule
Method & Deployed output & Comparisons at main setting \\
\midrule
Multigroup GLHF & Lottery depending on group memberships & Active \\
Batch NLHF & One learned lottery for everyone & Passive \\
Pooled GLHF & One learned lottery for everyone & Active \\
Uniform reference & Equal probability for every alternative & None \\
\bottomrule
\end{tabular}
\caption{Methods at the main setting. Each learner receives one comparison per training user. Passive versions of both GLHF methods are also evaluated. Batch NLHF always uses passive comparisons.}
\label{tab:empirical-methods}
\end{table}

The reference lottery is uniform over the four coffees or six answers. The learner radius $\tau_L$ limits a served lottery's KL divergence from this reference; the evaluation radius $\tau_E$ constrains the fixed-lottery welfare benchmark. Both radii are $0.5$ at the main setting. The sensitivity checks in Appendix~\ref{app:empirical-robustness} vary $\tau_L$ while keeping the benchmark fixed.

\paragraph{Generating training comparisons.}
Each training respondent supplies exactly one comparison per fit. Passive sampling draws an unordered pair uniformly. Active sampling draws the first alternative from the learner's current policy for that respondent and the second uniformly from the remaining alternatives. At the main setting, $\beta=\infty$, the higher-rated alternative wins and a seeded fair coin resolves ties. At finite $\beta$, the probability that $a$ wins against $b$ is $\sigma(\beta[U(a)-U(b)])$, as in Assumption~\ref{ass:bt}. Smaller $\beta$ makes labels less reliably follow rating order. The learner observes the queried pair and winner, together with group memberships, but not the ratings used to generate the label.

Within each training split and learner seed, methods share training respondents and their order. Passive multigroup GLHF, passive pooled GLHF, and NLHF share the same comparison log. Active multigroup and pooled GLHF each generate their own policy-dependent log, so equal comparison budgets do not imply identical queried pairs.

\paragraph{Fitting and deploying policies.}
Multigroup GLHF makes one pass through the training users, updating the specialists for every group containing the current user. It saves the policy before each comparison. To serve a new user, it averages the lotteries those saved policies assign to that user's memberships. Every group evaluates this same mapping. Pooled GLHF averages a single specialist's lotteries. Both use Algorithm~\ref{alg:main}'s adaptive learning rates; Appendix~\ref{app:empirical-rates} gives numerical details.

Batch NLHF estimates each pair's win probability from the complete passive comparison log and fits a constrained maximal lottery to the estimated margin matrix game. It uses no additional comparisons or held-out ratings. Appendix~\ref{app:empirical-nlhf} gives the estimator and numerical checks.

\subsubsection{Held-out welfare}
\label{app:empirical-evaluation}

Discovery rows are excluded from all reported policy fitting and evaluation. We randomly shuffle the remaining rows and divide them into ten nearly equal folds. We train on nine folds and test on the tenth, rotating through all ten test folds. We repeat this procedure with ten different random partitions, giving 100 train/test splits. Each split uses 50 learner seeds to vary the randomized training procedure.

For each held-out group, empirical distortion divides the welfare of its best fixed feasible lottery by the welfare of the deployed policy. For group $k$ in split $f$, let $I_{kf}$ be its test respondents, $n_{kf}=|I_{kf}|$, and $\bar U_{kf}$ their mean utility vector. The benchmark welfare is
\begin{equation}
\widehat W^*_{kf}
=\max_{q:\,\KL(q\|\piref)\leq\tau_E}q^\top\bar U_{kf}.
\label{eq:empirical-welfare-optimum}
\end{equation}
For method $j$ and learner seed $s$, define
\begin{equation}
\widehat{\Dist}^{(j)}_{kfs}
=\frac{\widehat W^*_{kf}}
{n_{kf}^{-1}\sum_{i\in I_{kf}}U_i^\top\pi^{(j)}_{fs}(X_i)},
\qquad
\widehat{\Dist}^{(j)}_{fs}=\max_k\widehat{\Dist}^{(j)}_{kfs}.
\label{eq:empirical-distortion}
\end{equation}

The benchmark is optimized on the test fold using full ratings; it is not another trained method. This optimization can make benchmark welfare optimistic relative to the population optimum on average. Held-out utilities never influence policy training, and the reported ratios describe welfare on these test folds.

\subsubsection{Averaging and uncertainty}
\label{app:empirical-inference}

For worst-group results, we take the maximum across groups within each split and seed, average the 50 seeds within each split, and then average the 100 splits. Thus the reported worst-group value is not the maximum of the displayed per-group means. For method comparisons, we first subtract the two methods' distortions within each split and seed, then follow the same averaging procedure. Positive NLHF-minus-multigroup, pooled-minus-multigroup, and reference-minus-multigroup differences favor multigroup GLHF; positive passive-minus-active differences favor active sampling.

The train/test splits overlap, so neither the 100 split scores nor the 5,000 seed--split values are independent replicates. We use the Nadeau--Bengio correction~\citep{nadeau2003} to form approximate 95\% $t$ intervals from the split scores, with 99 degrees of freedom and standard error
\[
\widehat{\rm SE}
=\sqrt{\left(\frac{1}{100}+\overline{n_{\rm test}/n_{\rm train}}\right)
S^2_{\rm split}}.
\]
Here $S^2_{\rm split}$ is the sample variance of the 100 seed-averaged scores, and the ratio is averaged over the actual split sizes. For a method comparison, the scores are paired differences. Overlap between two methods' individual intervals is not a test of their paired difference. The correction is approximate and does not establish exact coverage for the nonlinear worst-group ratio. These intervals condition on the selected groups and do not account for uncertainty from group selection or earlier dataset exploration. We also use two-sided tests for the per-group NLHF-minus-multigroup differences, with Holm adjustment~\citep{holm1979} across the six groups in each study.

\subsection{Results at the main setting}
\label{app:empirical-anchor}
\label{app:persona-results}

At $\beta=\infty$ and $\tau_L=\tau_E=0.5$, multigroup GLHF has lower mean worst-group distortion than NLHF in both studies: 1.08 versus 1.28 in the coffee-rating study and 1.07 versus 1.36 in the completion-rating study. Table~\ref{tab:empirical-anchor} gives the precise levels and paired differences. Both GLHF methods use active sampling here; NLHF remains passive-trained.

\begin{table}[!htbp]
\centering
\small
\setlength{\tabcolsep}{3pt}
\begin{tabular}{lccccccc}
\toprule
Study & MG & \shortstack{Batch\\NLHF} & Pooled & Ref. & \shortstack{Batch NLHF\\$-$ MG} & Pooled $-$ MG & Ref. $-$ MG \\
\midrule
\shortstack[l]{Coffee-rating\\study}
& 1.0795 & 1.2816 & 1.2528 & 1.1217
& \shortstack{0.2021\\$[0.1507,0.2535]$}
& \shortstack{0.1732\\$[0.1293,0.2171]$}
& \shortstack{0.0421\\$[0.0323,0.0519]$} \\
\shortstack[l]{Completion-rating\\study}
& 1.0695 & 1.3646 & 1.3501 & 1.7357
& \shortstack{0.2951\\$[0.2823,0.3080]$}
& \shortstack{0.2807\\$[0.2700,0.2913]$}
& \shortstack{0.6662\\$[0.6540,0.6784]$} \\
\bottomrule
\end{tabular}
\caption{Worst-group distortion at $\beta=\infty$ and $\tau_L=\tau_E=0.5$ (lower is better). GLHF uses active sampling; batch NLHF uses passive sampling. MG denotes multigroup GLHF and Ref. the uniform reference. Method columns are means; difference columns give paired estimates and corrected 95\% intervals over 100 splits, each averaging 50 learner seeds.}
\label{tab:empirical-anchor}
\end{table}

\paragraph{Groupwise welfare.}
In the coffee-rating study, multigroup GLHF has lower estimated distortion than NLHF and pooled GLHF in every selected group. The largest subgroup gaps are for respondents aged 45+, those with coffee expertise 1--4, and female respondents. It also has the lowest universal-group estimated distortion: $0.992$, compared with $1.005$ for NLHF and $1.012$ for pooled GLHF (Figure~\ref{fig:coffee-group-bars}). The figure includes the untrained uniform lottery as a reference.

In the completion-rating study, multigroup GLHF has lower estimated distortion than NLHF and pooled GLHF in every selected group (Figure~\ref{fig:persona-group-bars}). It exceeds the fixed-lottery benchmark for expressive wording, university education, and the universal group. These values below one reflect its ability to vary the served lottery with memberships.

\begin{figure}[!htbp]
\centering
\includegraphics[width=\linewidth]{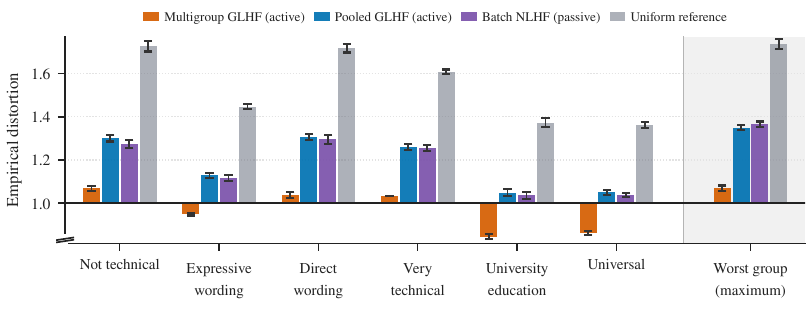}
\caption{Empirical distortion in the completion-rating study at $\beta=\infty$ and $\tau_L=\tau_E=0.5$: active multigroup GLHF, active pooled GLHF, passive batch NLHF, and the uniform reference. Lower is better. Bars below one indicate welfare above the group's best fixed feasible lottery. The rightmost cluster averages the maximum across groups within each run. Whiskers are corrected 95\% confidence intervals for individual means.}
\label{fig:persona-group-bars}
\end{figure}

\paragraph{Paired comparisons within groups.}
Figure~\ref{fig:empirical-group-gaps} compares multigroup GLHF with NLHF on each group. Multigroup distortion is lower in all six groups in each study, with every two-sided test significant at the 0.05 level after Holm adjustment. This conclusion concerns NLHF at the main setting; the group-bar point estimates alone do not establish significance against every comparator.

\begin{figure}[htbp]
\centering
\includegraphics[width=\linewidth]{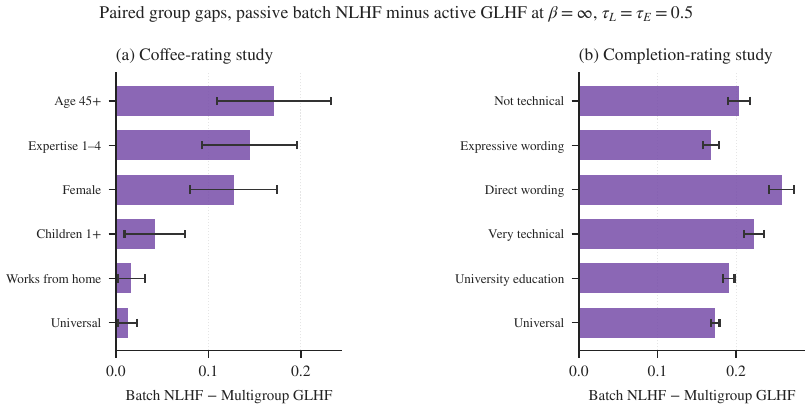}
\caption{Per-group distortion differences: passive batch NLHF minus active multigroup GLHF at $\beta=\infty$, $\tau_L=\tau_E=0.5$. Positive values favor multigroup GLHF. Corrected 95\% intervals apply to each difference separately, not to all differences jointly. The accompanying two-sided tests use Holm adjustment over the six groups within each study.}
\label{fig:empirical-group-gaps}
\end{figure}

\paragraph{Active and passive sampling.}
Section~\ref{sec:empirical-trajectories} reports the paired endpoint comparisons and Figure~\ref{fig:empirical-learning-curves} shows the trajectories. Both use the held-out evaluation and averaging procedure above, with group maxima taken within each split and seed before averaging. The plotted confidence intervals apply to individual checkpoints.

\subsection{Sensitivity checks}
\label{app:empirical-robustness}

\paragraph{Learner radius.}
A larger $\tau_L$ allows the learner to move farther from the uniform reference. Figures~\ref{fig:coffee-radius} and~\ref{fig:persona-radius} vary $\tau_L$ over $\{0.05,0.2,0.35,0.5,0.75,1\}$ while keeping the evaluation benchmark fixed at $\tau_E=0.5$ and feedback at $\beta=\infty$. Both GLHF methods use active sampling; NLHF remains passive-trained.

Multigroup GLHF has the lowest mean worst-group distortion among the trained methods at every tested radius in both studies. In the coffee-rating study, its distortion decreases from 1.110 to 1.072 as the radius grows, while pooled GLHF and NLHF worsen. For NLHF, universal-group distortion falls with the larger radius, but distortion rises for respondents aged 45+, with coffee expertise 1--4, and female respondents; these groups' pairwise preference margins were negatively correlated with the population margins in the discovery sample. This pattern is consistent with a single lottery moving toward coffees favored by the population overall while serving groups with less aligned preferences less well. In the completion-rating study, multigroup distortion falls from 1.452 to 0.953. Because the benchmark radius stays fixed, these curves show how additional training flexibility changes welfare against the same target. When $\tau_L>\tau_E$, distortion below one can also reflect the learner's larger feasible set.

\begin{figure}[!htbp]
\centering
\includegraphics[width=\linewidth]{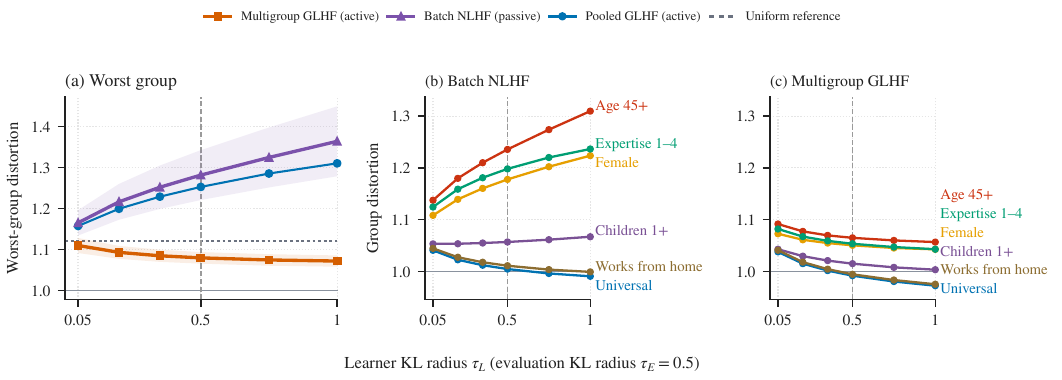}
\caption{Distortion in the coffee-rating study as the learner radius varies, with $\beta=\infty$ and evaluation radius fixed at $\tau_E=0.5$. The evaluated learner KL radii are $0.05$, $0.2$, $0.35$, $0.5$, $0.75$, and $1$. Both GLHF methods use active sampling; batch NLHF uses passive sampling. Panel (a) shows worst-group distortion for all four methods; its dotted horizontal line marks the uniform reference. Panels (b) and (c) show groupwise batch NLHF and multigroup GLHF distortion. The dashed line marks the main learner radius, $0.5$. Bands in (a) are corrected 95\% intervals for individual means; groupwise curves show point estimates.}
\label{fig:coffee-radius}
\end{figure}

\begin{figure}[!htbp]
\centering
\includegraphics[width=\linewidth]{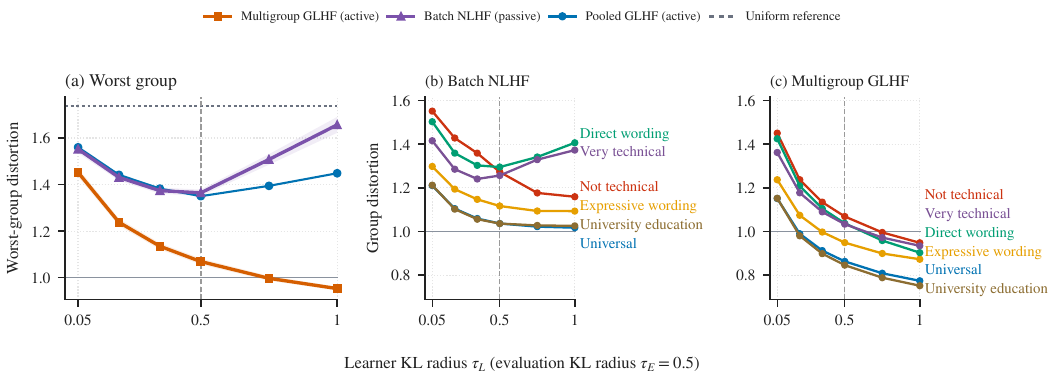}
\caption{Distortion in the completion-rating study as the learner radius varies, with $\beta=\infty$ and evaluation radius fixed at $\tau_E=0.5$. The evaluated learner KL radii are $0.05$, $0.2$, $0.35$, $0.5$, $0.75$, and $1$. Both GLHF methods use active sampling; batch NLHF uses passive sampling. Panel (a) shows worst-group distortion for all four methods; its dotted horizontal line marks the uniform reference. Panels (b) and (c) show groupwise batch NLHF and multigroup GLHF distortion. Bands in (a) are corrected 95\% intervals for individual means; groupwise curves show point estimates.}
\label{fig:persona-radius}
\end{figure}

\paragraph{Comparison noise.}
Section~\ref{sec:empirical-noise} reports the full noise grid. From $\beta=1/2$ to $\infty$, passive multigroup distortion decreases from 1.125 to 1.091 in the coffee-rating study and from 1.530 to 1.099 in the completion-rating study. At $\beta=4$, active sampling lowers the multigroup point estimates relative to passive sampling by 0.0113 and 0.0431, respectively.

\paragraph{Constructing pseudo-respondents.}
We repeated the completion-rating analysis with four additional shuffles of the same ratings, using passive sampling at $\beta=\infty$ and $\tau_L=\tau_E=0.5$. We repeated group selection for each shuffle and selected the same five subgroups every time. The pooled-minus-multigroup differences in worst-group distortion ranged from $0.2747$ to $0.2786$, close to $0.2760$ for the original construction. The active results use only the original construction.

\subsection{Limitations}
\label{app:empirical-limitations}

These studies evaluate welfare under specific rating-based utility models and selected group definitions. Both assume that equal rating increments represent equal, interpersonally comparable utility increments; the reported distortion also depends on assigning zero utility to the lowest rating. The coffee study captures heterogeneous preferences over shared alternatives, but choosing among four coffees does not test the complexity of LLM responses or variation across prompts. The completion study uses one model to rate six answers to one question, with separately elicited ratings combined into pseudo-respondents; it does not measure human preferences. Groups were chosen partly to capture disagreement with population preferences, and the confidence intervals condition on these choices rather than accounting for variability in group selection. The baseline comparisons therefore demonstrate gains from conditioning on these memberships, but do not establish an advantage over other group-aware learning methods.

At the deterministic-comparison setting ($\beta=\infty$), the worst-case welfare factor $L/\ell_\beta$ diverges; low empirical distortion in this setting does not validate a finite theoretical welfare bound. Section~\ref{sec:empirical-noise} also evaluates finite $\beta$, generating labels from the individual Bradley--Terry model on the rescaled ratings. The reported ratios are empirical welfare measurements, not evaluations of the theoretical worst-case distortion bound.

\subsection{Implementation and reproducibility}

\subsubsection{Adaptive learning rates}
\label{app:empirical-rates}

The implementation uses the rates and updates in Section~\ref{sec:algorithm}, with $\tau=\tau_L$ and $T$ equal to the number of training comparisons. Rates depend on accumulated comparison scores, not population group masses or the confidence level. Each saved policy uses the state before its round's comparison, as in Algorithm~\ref{alg:main}. KL projections are computed numerically as described below. At $\tau_L=0$, the output is the reference lottery and no updates are needed.

\paragraph{Computing the KL projection.}
We compute the projection of $z_{t,k}$ using
Appendix~\ref{app:projection}, with $v=z_{t,k}$ and
$\tau=\tau_L$. If the candidate lies within the learner's
KL ball, we retain it. Otherwise, we use Brent's method
to find the point on the geometric path from $\piref$
to $z_{t,k}$ whose KL divergence from $\piref$ equals
$\tau_L$.

The numerical code treats radii at most $10^{-10}$ as zero, accepts a
candidate with KL divergence at most $\tau_L+10^{-10}$, and uses absolute
and relative root-search tolerances of $10^{-12}$. It also floors
underflowed probabilities at the smallest positive normal floating-point
value. The theorem analyzes exact updates and projections; the reported
implementation uses these numerical approximations.

\paragraph{Pooled implementation.}
For pooled GLHF, $K=1$ and the only specialist serves everyone: $\pihat_t(x)=\pi_{t,1}$. Score cancellation gives $\pi_{t,1}^{\top}\ghat_t=0$, so the outer weight stays at one. The adaptive schedule remains well defined at $K=1$ because $\Lambda_T\ge1$.

\subsubsection{Computing the welfare benchmark}
\label{app:empirical-kl-linear}

To compute the benchmark in
\eqref{eq:empirical-welfare-optimum}, we maximize
$q^\top\bar U_{kf}$ over $\B_{\tau_E}$.
At zero radius or when all alternatives have equal mean
utility, we use $\piref$. If the reference lottery
conditioned on the highest-utility alternatives is feasible,
we use that lottery. Otherwise, the KL constraint binds,
and the solution has the exponential-tilt form
\[
q_\lambda(a)
=\frac{\piref(a)\exp(\bar U_{kf}(a)/\lambda)}
{\sum_b\piref(b)\exp(\bar U_{kf}(b)/\lambda)},
\qquad \lambda>0.
\]
Here $\lambda$ is the multiplier for the KL constraint.
We use a one-dimensional root search to choose it so that
$\KL(q_\lambda\|\piref)=\tau_E$. Benchmark and deployed welfare
are calculated as expectations under their respective
lotteries, without sampling an alternative.

\subsubsection{Batch NLHF}
\label{app:empirical-nlhf}

For each pair $(a,b)$, NLHF estimates $\Pr(a\succ b)$ as
the observed fraction of wins for $a$ among comparisons
of that pair. An unobserved pair is assigned probability
$1/2$ and recorded as unobserved. These estimates define
the antisymmetric margin matrix
$\widehat M_{ab}=2\widehat{\Pr}(a\succ b)-1$ for $a\ne b$,
with $\widehat M_{aa}=0$. This estimates the population
margins $\E[M(X)]$ from the passive comparison log.
Using the feasible set from \eqref{eq:ball} with radius
$\tau_L$, we seek a fixed lottery $p$ minimizing the largest
advantage of any feasible challenger:
\[
\min_{p\in\B_{\tau_L}}\max_{q\in\B_{\tau_L}}
q^\top\widehat M p.
\]
The fitted policy serves this same lottery to every user.
Antisymmetry and the common feasible set imply that
the optimal value is zero.

For a fixed lottery $p$, the dual formulation is
\[
\max_{q\in\B_{\tau_L}}q^\top\widehat M p
=\inf_{\lambda>0}
\left\{
\lambda\tau_L+
\lambda\log\left[
\sum_a\piref(a)
\exp\left(\frac{(\widehat M p)_a}{\lambda}\right)
\right]
\right\}.
\]
Here $\lambda$ is the multiplier for the challenger's KL
constraint. We minimize the right-hand expression jointly
over $p\in\B_{\tau_L}$ and the auxiliary variable $\log\lambda$
using SciPy's SLSQP solver, with finite-difference
derivatives, at most 500 iterations, and
\texttt{ftol}$=10^{-12}$. Each fit uses one optimization
call initialized at $p=\piref$ and
$\lambda=\max\{\tau_L,10^{-3}\}$, without restarts or
a secondary equilibrium-selection step.

We normalize the returned lottery and, if needed,
project it back into the learner's KL ball. We then
check its maximum feasible challenger advantage using
the exponential-tilt procedure described in
Appendix~\ref{app:empirical-kl-linear}, replacing
$\bar U_{kf}$ by $\widehat M p$ and $\tau_E$ by
$\tau_L$. We accept a fit only if this advantage is
at most $2\times10^{-5}$; all reported fits passed.
This check is separate from the optimizer's convergence
flag and does not concern held-out welfare.

\subsubsection{Synthetic collection materials}
\label{app:persona-materials}

The collection requests specified Claude Haiku 4.5, temperature 1.0, a maximum output length of 64 tokens, and extended thinking disabled. Each request contained a single user message with one persona background, the question, and one answer, with no conversation history or other candidate answers. It requested a JSON object containing one integer score. Background lines appeared in the fixed order age, education, technical background, and preferred writing style. Every combination of the four two-level attributes in Table~\ref{tab:persona-features} was used, giving 16 personas.

\begin{table}[!htbp]
\centering
\small
\begin{tabular}{p{0.18\linewidth}p{0.36\linewidth}p{0.36\linewidth}}
\toprule
Feature & First level & Second level \\
\midrule
Age & 22-32 & 62-72 \\
Education & Highest completed education is high school or less & Highest completed education is a university degree \\
Technical background & Not technical; no technical/STEM work, training, or study & Very technical; extensive technical/STEM work, training, or study \\
Preferred writing style & Prefers direct and literal wording & Enjoys expressive and literary wording \\
\bottomrule
\end{tabular}
\caption{The four two-level attributes used to form the 16 persona descriptions.}
\label{tab:persona-features}
\end{table}

\paragraph{Prompt template.}
The following text is reproduced from the collection. The placeholders are replaced with the persona background, the question \emph{Why is the sky blue?}, and exactly one answer below.
\begin{quote}
\small
Use the following background as context:

\noindent \{background\}

\noindent People in these groups can disagree. Do not assume any traits or personal history beyond what is listed.

\noindent You just asked, "\{question\}" One person gave you the answer below, exactly as written, with no further explanation:

\noindent \{answer\}

\noindent Rate how well the answer works for you.

\noindent Give one overall score from 1 to 10. A score of 1 means it does not work for you at all; 5 means it partly works but has problems; and 10 means it works extremely well for you.

\noindent Return only a JSON object with one key named "score" and an integer from 1 through 10 as its value. Do not provide an explanation.
\end{quote}

\paragraph{Fixed answer set.}
The first four answers vary accessible versus technical content and direct versus expressive style. The last two are negative controls. The headings below identify each answer's role in the study; they were not shown to the evaluator.

\paragraph{Accessible, direct (correct).}
Sunlight contains many colors, and the air spreads the blue ones around more than the red ones. That is why the sky looks blue.

\paragraph{Accessible, expressive analogy (correct).}
Imagine white sunlight as a crowd of colors crossing the air. The red colors mostly keep marching forward, while the blue colors are nudged into wandering paths by tiny molecules. Those blue wanderers reach your eyes from all across the heavens, painting the daytime view blue. It is a simple picture of how air scatters light.

\paragraph{Technical, direct (correct).}
How strongly air redirects light depends sharply on the light's wavelength: the strength varies approximately with the inverse fourth power, so 450-nanometer blue light is scattered several times more strongly than 650-nanometer red light. This is the Rayleigh regime, which applies because air molecules are much smaller than visible wavelengths. Scattering along every line of sight creates a blue-weighted diffuse radiance. Violet is less prominent because the solar spectrum, atmospheric absorption, and human visual sensitivity reduce its contribution. At sunset, the longer optical path scatters blue out of the direct beam, leaving red and orange.

\paragraph{Technical, expressive (correct).}
Air does not redirect every color equally. Its effect rises sharply as wavelength shortens, roughly with the inverse fourth power, so 450-nanometer blue light is scattered several times more strongly than 650-nanometer red light. In this Rayleigh regime, blue photons are swept from the Sun's direct beam and distributed across innumerable lines of sight, building the blue radiance overhead; red photons travel onward more faithfully. Violet scatters still more, yet solar output, atmospheric absorption, and the eye's sensitivity mute it. At sunset, the lengthened path spends the blues and leaves the horizon burning red.

\paragraph{Fictional-story control (incorrect).}
When the world was new, the sky was clear as glass and the sea alone held color. Each morning, the wind carried a little of the sea upward, spreading its blue across the empty air so travelers could tell day from night. Rain returned some of that color to the water, while evening loosened the rest into red and gold before darkness arrived. That is why the sky and ocean share one blue, continually passed between them by wind and weather.

\paragraph{Plain incorrect control (incorrect).}
The sky is blue because it reflects the color of the oceans below it. Light bounces off the water and spreads across the sky.

\subsubsection{Code and saved results}
\label{app:empirical-code}

The \href{https://github.com/Rob10001/groupwise-distortion-experiments}{public GitHub repository} contains the code, both input datasets, saved evaluation results, and figure-generation scripts. Its \nolinkurl{README.md} gives installation and reproduction commands, and \nolinkurl{data/README.md} describes the inputs. Paths below are relative to the repository root:
\begin{itemize}[leftmargin=1.5em,itemsep=1pt,topsep=3pt]
\raggedright
\item Input datasets: \nolinkurl{data/coffee/} and \nolinkurl{data/persona/}; the latter also contains the synthetic collection materials.
\item Comparison results: \nolinkurl{results/coffee/comparison/} and \nolinkurl{results/persona/comparison/}; each dataset's \nolinkurl{beta-source/} directory contains the $\beta$ runs.
\item Learning-curve reductions: \nolinkurl{results/coffee/learning-curves/} and \nolinkurl{results/persona/learning-curves/}.
\item Synthetic-row shuffles: \nolinkurl{results/persona/panel-permutations/}.
\end{itemize}

Saved comparison results include held-out welfare and distortion for each split, seed, method, and group. The empirical package retains these results and their reduction tables, while omitting per-comparison training logs and per-split learning-curve arrays to reduce its size. The figure scripts can draw the paper figures from the bundled tables.

Clone or download the repository and run commands from its root directory. Using Python 3.12, install dependencies and reproduce the experiments with:
\begin{verbatim}
python -m pip install -r requirements.txt
python reproduce.py --workers 4
\end{verbatim}
This reruns both datasets' $\beta$/protocol grids, GLHF and NLHF comparisons, learner-radius sweeps, learning curves, and five synthetic-row alignments from the input ratings, then generates the paper figures. New results are written to \nolinkurl{results/reproduced/}. For a reduced end-to-end check, run:
\begin{verbatim}
python reproduce.py --smoke --workers 2
\end{verbatim}
The README also gives commands for individual experiments and for regenerating figures from saved results.

\end{document}